%% file: Main.tex
\documentclass[anon,12pt]{article} 

\title{Over-Parameterization Exponentially Slows Down Gradient Descent for Learning a Single Neuron}
\usepackage{times}
\usepackage{settings}

\author[1]{Weihang Xu\thanks{xuwh19@mails.tsinghua.edu.cn. Part of this work was done while W. Xu was visiting University of Washington.} }
\author[2]{Simon S. Du\thanks{ssdu@cs.washington.edu}}
\affil[1]{\small Institute for Interdisciplinary Information Sciences, Tsinghua University. }
\affil[2]{\small Paul G. Allen School of Computer Science and Engineering, University of Washington}

\begin{document}

\maketitle

\begin{abstract}%
We revisit the problem of learning a single neuron with ReLU activation under Gaussian input with square loss.
We particularly focus on the over-parameterization setting where the student network has $n\ge 2$ neurons.
We prove the global convergence of randomly initialized gradient descent with a $O\left(T^{-3}\right)$ rate.
This is the first global convergence result for this problem beyond the exact-parameterization setting ($n=1$) in which the gradient descent enjoys an $\exp(-\Omega(T))$ rate.
Perhaps surprisingly, we further present an $\Omega\left(T^{-3}\right)$ lower bound for randomly initialized gradient flow in the over-parameterization setting.
These two bounds jointly give an exact characterization of the convergence rate and imply, for the first time, that \emph{over-parameterization can exponentially slow down the convergence rate}.
To prove the global convergence, we need to tackle the interactions among student neurons in the gradient descent dynamics, which are not present in the exact-parameterization case.
We use a  three-phase structure to analyze GD's dynamics. Along the way, we prove gradient descent automatically balances student neurons, and use this property to deal with the non-smoothness of the objective function. To prove the convergence rate lower bound, we construct a novel potential function that characterizes the pairwise distances between the student neurons (which cannot be done in the exact-parameterization case).
We show this potential function converges slowly, which implies the slow convergence rate of the loss function.

\end{abstract}

\input{Content/Intro.tex}

\input{Content/MainResult.tex}


\bibliography{ref}

\newpage
\appendix

\input{Content/UpperBound.tex}

\input{Content/LowerBound.tex}

\end{document}

%% file: Content/Intro.tex
\section{Introduction}
In recent years, theoretical explanations of the success of gradient descent (GD) on training deep neural networks emerge as an important problem. A prominent line of work \cite{AllenZhu18Beyond, Du18provably, Jacot18NTK, safran18spurious, Chizat18Lazy} suggests that over-parameterization plays a key role in the successful training of neural networks. 

However, the drawback of over-parameterization is under-explored. In this paper, we consider training two-layer ReLU networks, with a particular focus on learning a single neuron in the over-parameterization setting.
We give a rigorous proof for the following surprising phenomenon:
\begin{center}
\emph{Over-parameterization exponentially slows down the convergence of gradient descent.}
\end{center}


Specifically, we consider two-layer ReLU networks with $n$ neurons and input dimension $d$:
\begin{equation}\label{network}
    \x\to \sum_{i=1}^n [\w_i^\top\x]_+,
\end{equation}
where $[x]_+= \max\{0, x\}$ denotes the ReLU function, $\w_1, \ldots, \w_n\in\mathbf{R}^d$ are $n$ neurons. The input $\x\sim\mathcal{N}(0, I)$ follows a standard Gaussian distribution. 

We consider the teacher-student setting, where a student network is trained to learn a ground truth teacher network. Following the architecture \eqref{network}, the student network $f:\mathbf{R}^d\to \mathbf{R}$ is given by $f(\x)=\sum_{i=1}^n[\w_i^\top\x]_+,$ where $\w_1, \ldots, \w_n\in\mathbf{R}^d$ are $n$ student neurons. Similarly, the teacher network is given by $f^*(\x)=\sum_{i=1}^m[\vv_i^\top\x]_+$, where $\vv_1, \ldots, \vv_m\in\mathbf{R}^d$ are $m$ teacher neurons. It is natural to study the square loss:
\begin{equation}\label{general loss}
    L(\w)=\mathbb{E}_{\x\sim\mathcal{N}(\bm{0}, I)}\left[\frac{1}{2}\left(\sum_{i=1}^n[\w_i^\top\x]_+ -\sum_{i=1}^m[\vv_i^\top\x]_+\right)^2\right],
\end{equation}
where $\w=(\w_1^\top, \w_2^\top, \ldots, \w_n^\top)^\top\in \mathbf{R}^{n\times d}$ denotes the parameter vector formed by student neurons.

In this paper, we focus on the special case where the teacher network consists of one single neuron $\vv_1$, i.e., $m=1$. For simplicity, we omit the subscript and denote $\vv_1$ with $\vv$. 
The student network is initialized with a Gaussian distribution: $\forall 1\leq i\leq n, \w_i(0)\sim \mathcal{N}(\bm{0}, \sigma^2 I)$,  ($\sigma \in \mathbf{R}^+$ denotes the initialization scale), then trained by gradient descent with a step size $\eta$. 

\begin{figure}[t]
    \begin{minipage}{0.45\linewidth}
    \centering
    \includegraphics[width=\linewidth]{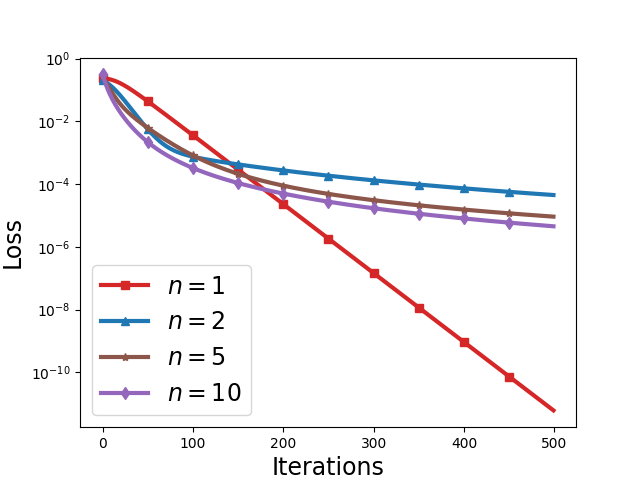}
         \label{single vs multi neurons.png}
     \end{minipage}
     \hfill
     \begin{minipage}{0.45\linewidth}
     \centering
        \includegraphics[width=\linewidth]{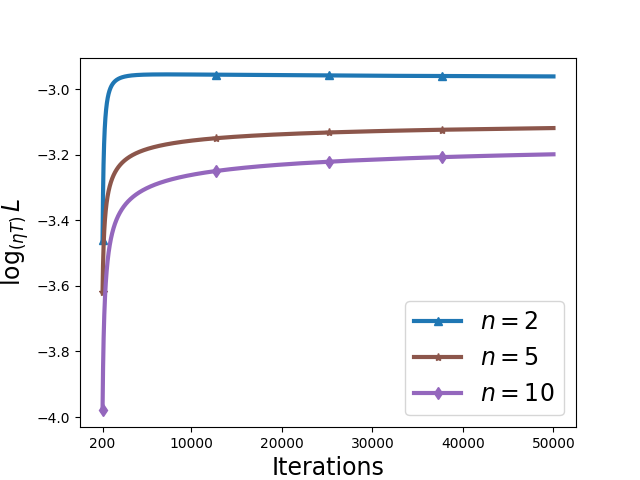}
         \label{loss-asymptotic.png}
     \end{minipage}
     \label{simulation}
     \caption{
     Setting: $\sigma =0.1, \eta = 0.05, \|\vv\|=1$. Left: The loss converges much slower when $n>1$, compared to the case of $n=1$. Right: $\log_{(\eta T)} L(\w(T))$ converges to $-3$, with a  small perturbation that converges extremely slow (note if  we want $\log_{(\eta T)}\frac{C}{(\eta T)^3}\in (-3-\epsilon, -3+\epsilon)$, then $ T\geq \frac{1}{\eta}C^{1/\epsilon}$ is needed.)}

\end{figure}

In this widely-studied setting, we discover a new phenomenon: compared to the exact-parameterized case ($n=1$), the loss $L(\w(t))$ converges much slower in the over-parameterized case. Empirically (see Figure \ref{simulation}), the slow-down effect happens universally for all $n\geq 2$. Moreover, $\log_{(\eta T)}L(\w(T))$ has a tendency of converging towards $-3$, which seems to suggest that the convergence rate should be $L(\w(T))=\Theta( T^{-3})$. 

For the exact-parameterized case ($n=1
$), \citet{Yehudai20SingleNeuron} proved that $L(\w(t))$ convergences with a linear rate: $L(\w(t)) \leq \exp\left(-\Omega(t)\right)$, which is also validated in Figure \ref{simulation}. For the over-parameterized case, an exact characterization of the convergence rate is given in this paper as $L(\w(t)) = \Theta(t^{-3})$.
As a result, we show that (even very mild) over-parameterization exponentially slows down the convergence rate.
Specifically, our main results are the following two theorems.

\begin{theorem}[Global Convergence, Informal]\label{informal global convergence}
    For $\forall \delta >0$, suppose the dimension $d=\Omega(\log(n/\delta))$, the initialization scale \footnote{Note that $\|\w_i(0)\|$ scales with $\sigma\sqrt{d}$ rather than $\sigma$. } $\sigma\sqrt{d}= poly(n^{-1})\|\vv\|,$ the learning rate $\eta =poly(\sigma\sqrt{d}, n^{-1}, \|\vv\|^{-1})$. Then with probability at least $1-\delta$, gradient descent converges to a global minimum with rate $L(\w(t)) \leq  poly(n, \|\vv\|, \eta^{-1})t^{-3}$.
\end{theorem}
\begin{theorem}[Convergence Rate Lower Bound, Informal]\label{Informal Lower Bound}
Suppose the student network is over-parameterized, i.e., $n\geq 2$. Consider gradient flow: $\frac{\partial \w(t)}{\partial t} = -\frac{\partial L(\w(t))}{\partial\w}.$
    If the requirements on $d$ and $\sigma$ in Theorem \ref{informal global convergence} hold, then with high probability, there exist  constants $\Gamma_1, \Gamma_2$ which do not depend on time $t$, such that $\forall t\geq 0, L(\w(t))\geq {(\Gamma_1 t+\Gamma_2)^{-3}}$.
\end{theorem}

Theorem \ref{informal global convergence} shows the global convergence of GD, while Theorem \ref{Informal Lower Bound} provides a convergence rate lower bound. These two bounds together imply an exact characterization of the convergence rate for GD.
We further highlight the significance of our contributions below:

\noindent $\bullet$ 
To our knowledge, Theorem~\ref{informal global convergence} is the first global convergence result of gradient descent for the square loss beyond the special exact-parameterization cases of $m=n=1$~\citep{tian2017analytical,brutzkus2017globally,Yehudai20SingleNeuron,du2017convolutional} and $m=n=2$~\citep{wu2018no}.

\noindent $\bullet$ While over-parameterization is well-known for its benefit in establishing global convergence in the finite-data regime, this is the first work proving it can slow down gradient-based methods. 

\subsection{Related Works}
The problem of learning a single neuron is actually well-understood and can be solved with minimal assumptions by classical single index models algorithms \citep{kakade2011efficient}. For learning a single-neuron, \citet{brutzkus2017globally,tian2017analytical,soltanolkotabi2017learning} proved convergence for GD assuming Gaussian input distribution, which was later improved by  \citet{Yehudai20SingleNeuron} who proved linear convergence of GD for learning one single neuron properly.
These results are also generalized to learning a convolutional filter~\citep{goel2018learning,du2017convolutional,du2018gradient,zhou2019toward,liu2019towards}. 
These works only focus on the  exact-parameterization setting, while we focus on the over-parameterization setting.

Another direction focuses on the optimization landscape.
\citet{safran18spurious} showed spurious local minima exists for large $m$ in the exact-parameterization setting.
\citet{safran20effects} studied problem \eqref{general loss} with orthogonal teacher neurons. They showed that neither one-point strong convexity nor Polyak-Łojasiewicz (PL) condition hold locally near the global minimum. \citet{wu2018no} showed that problem \eqref{general loss} has no spurious local minima for $m=n=2$. \citet{zhong2017recovery,zhang2019learning} studied the exact-parameterization setting and showed the local strong convexity of loss and therefore with tensor initialization, GD can converge to a global minimum.
\citet{arjevani2022annihilation} proved that over-parameterization annihilates certain  types of spurious local minima.

A popular line of works, known as neural tangent kernel (NTK) \citep{Jacot18NTK,Chizat18Lazy, Du18provably,du2019gradient, cao2019generalization, allen2019convergence, arora2019fine, oymak2020toward,zou2020gradient,li2018learning} connects the training of ultra-wide neural networks with kernel methods. Another line of works uses the mean-field analysis to study the training of infinite-width neural networks \citep{nitanda2017stochastic, chizat2018global, wei2019regularization, nguyen2020rigorous, fang2021modeling, lu2020mean}. All of these works considered the finite-data regime and require the neural network to be ultra-wide, sometimes infinitely wide. Their techniques cannot explain the learnability of a single neuron, as pointed out by \cite{yehudai2019power}.

More related to our works are results on the dynamics of gradient descent in the teacher-student setting. \citet{li2017convergence} studied the exact-parameterized setting and proved convergence for SGD with initialization in a region near identity.
\citet{li2020learning} showed that GD can learn two-layer networks better than any kernel methods, but their final upper bound of loss is constantly large and no convergence is proven.
\citet{zhou2021local} proved \emph{local} convergence for mildly over-parameterized two-layer networks. While our global convergence analysis uses their idea of establishing a gradient lower bound, we also propose new techniques to get rid of their architectural modifications, and improved their gradient lower bound to yield a tight convergence rate upper bound (see Section \ref{Technical Overview} for details).
Also, \citet{zhou2021local} only provided a local convergence theory, while we prove convergence globally.
On the other hand, their results hold for general $m\ge1$ whereas we only study $m=1$.

The first phase of our analysis is similar to the initial alignment phenomenon in \cite{boursier2022gradient}. Their analysis also relies on the finite-data regime and the orthogonality of inputs, hence does not apply to our setting.


Similar slow-down effects of over-parameterization on the convergence rate have been observed in other scenarios.
\citet{richert2022soft} considered error function activation and empirically observed an $O(T^{-2})$ convergence rate.
Going beyond neural network training, \citet{dwivedi2020singularity,wu2019randomly} showed such a phenomenon for Expectation-Maximization (EM) algorithm on Gaussian mixture models. \citet{zhang2022preconditioned} exhibited similar empirical behaviors of GD on Burer–Monteiro factorization, but no rigorous proof was given.

\noindent \textbf{Paper Organization.}
In Section \ref{Technical Overview} we describe the main technical challenges in our analysis, and our ideas for addressing them. In section \ref{Preliminaries} we define some notations and preliminary notions. In Section \ref{Proof Sketch} we formalize the global convergence result (Theorem \ref{informal global convergence}) and provide a proof sketch.
In Section \ref{Proof Sketch: Convergence Rate Lower Bound} we formalize the convergence rate lower bound (Theorem \ref{Informal Lower Bound}) and provide a proof sketch.

\section{Technical Overview}\label{Technical Overview}


\noindent \textbf{Three-Phase Convergence Analysis. }
Our global convergence analysis is divided into three phases. We define $\theta_i$ as the angle between $\w_i$ and $\vv$, and $H\coloneqq \|\vv\|-\sum_i \langle \w_i, \overline{\vv}\rangle$. Intuitively, $\theta_i$ represents the radial difference between teacher and students, while $H$ represents the tangential difference between teacher and students.

When the initialization $\sigma\sqrt{d}$ is small enough, in phase $1$, for every $i\in[n]$, $\theta_i$ decreases to a small value while $\|\w_i\|$ remains small. In phase $2$, $\forall i\in[n], \theta_i$ remains bounded by a small value while $H$ decreases with an exponential rate. Both $\theta_i$ and $H$ being small at the end of phase $2$ implies that GD enters a local region near a global minimum. In phase $3$, we establish the local convergence by proving two properties: a lower bound of gradient, and a regularity condition of student neurons.

\noindent\textbf{Non-Benign Optimization Landscape.} Compared to the exact-parameterization setting, the optimization landscape becomes significantly different and much harder to analyze when the network is over-parameterized. \citet{zhou2021local} provided an intuitive illustration for this in their Section 4.
For the general problem \eqref{general loss}, \citet{safran20effects} showed that nice geometric properties that hold when $m=n$, including one-point strong convexity and PL condition, do not hold when $m<n$. In this paper, we go further and show that the difference in geometric landscape leads to totally different convergence rates.


\noindent\textbf{Non-smoothness and Implicit Regularization. } The loss function is not smooth when student neurons are close to $\bm{0}$, which brings
a major technical challenge for a local convergence analysis. \citet{zhou2021local} reparameterized the student neural network architecture to make the loss $L$ smooth. 
We show this artificial change is not necessary. 
Our observation is that GD implicitly regularizes the student neurons and keeps them away from the non-smooth regions near $\bm{0}$. To prove this, we show that $\w_i$ cannot move too far in phase 3, by applying an algebraic trick to upper-bound $\sum_{t = T}^{\infty} \eta\|\nabla_{\w_i}L(\w(t))\|$ with $L(\w(T))$ (Lemma \ref{locality}).
A similar regularization property for GD was given in \cite{du2018algorithmic}, but it applies layer-wise rather than neuron-wise as in our paper.

\noindent\textbf{Improving the Gradient Lower Bound. }
In our local convergence phase, we establish a local gradient lower bound similar to Theorem 3 in \cite{zhou2021local}. Moreover, we improve their bound from $\|\nabla_{\w} L(\w)\|\geq \Omega(L(\w))$ to $\|\nabla_{\w} L(\w)\|\geq \Omega(L^{2/3}(\w))$ (Theorem \ref{gdbound}). The idea in \cite{zhou2021local} is to pick an arbitrary global minimum $\{\w_i^*\}_{i=1}^n$ and show $\sum_i \langle \nabla_{\w_i}L(\w), \w_i-\w_i^* \rangle\geq L(\w)$. We improve their proof technique by carefully choosing a specific $\{\w_i^*\}_{i=1}^n$ such that $\|\w_i-\w_i^*\|$ is small, then applying Cauchy inequality to get a tighter bound.
This improvement is crucial since it improves the final bound of convergence rate from $L=O(T^{-1})$ in \cite{zhou2021local} to $L=O(T^{-3})$, which matches the lower bound in Theorem \ref{Informal Lower Bound}. This also indicates the optimality of the improved dependency $L^{2/3}$.

\noindent\textbf{Non-degeneracy Condition. }
While the lower bound for the convergence rate is straightforward to prove in the worst-case (i.e., from a bad initialization), the average-case (i.e., with random initialization) lower bound is highly-nontrivial due to the existence of several counter-examples in the benign cases (see Section \ref{Case Study}). 
To distinguish these counter-examples from general cases, we establish a new non-degeneracy condition and build our lower bound upon it. We define a potential function $Z(t)=\sum_{i<j} \|\z_i(t)-\z_j(t)\|$, where $\z_i:=\w_i-\langle\w_i, \overline{\vv}\rangle\overline{\vv}$. As long as the initialization is non-degenerate (See Definition \ref{non-degenerate}), then $Z(t)=\Omega(t^{-1})$ and $L(\w(t))\geq \Omega(Z^{3}(t)n^{-5}/\|\vv\|)$, which imply $L(\w(t))\geq \Omega(t^{-3})$. Intuitively, the slow convergence rate of $L$ when $n\geq 2$ is due to the slow convergence of term $\z_i-\z_j, (i\neq j)$, and we define $Z(t)$ to formalize this idea.

%% file: Content/MainResult.tex
\section{Preliminaries}\label{Preliminaries}

\textbf{Notations. } In this paper, bold-faced letters denote vectors. We use $[n]$ to denote $\{1,2,\ldots, n\}$.
For any nonzero vector $\vv\in \mathbf{R}^d$, the corresponding normalized vector is denoted with $\overline{\vv}:=\frac{\vv}{\|\vv\|}$. For two nonzero vectors $\w, \vv\in\mathbf{R}^d$, $\theta(\w,\vv)\coloneqq\arccos\left(\langle\overline{\w},\overline{\vv}\right\rangle)$ denotes the angle between them. 

For simplicity, we also adopt some notational conventions. Denote the gradient of the $i^{\text{th}}$ student neuron with $\nabla_i \coloneqq \frac{\partial L(\w)}{\partial \w_i}$. For any variable $\w$ that changes during the training process, $\w(t)$ denotes its value at the $t^{\text{th}}$ iteration, $e.g.$, $\w_i(t)$ indicates the value of $\w_i$ at the $t^{\text{th}}$ iteration. Sometimes we omit the iteration index $t$ when this causes no ambiguity. 
We abbreviate the expectation taken w.r.t the standard Gaussian as $\mathbb{E}_{\x}[\cdot]\coloneqq \mathbb{E}_{\x\sim\mathcal{N}(\bm{0}, I)}[\cdot]$.

\noindent\textbf{Special Notations for Important Terms.}
There are several important terms in our analysis and we give each of them a special notation.
$\theta_{i}\coloneqq\theta(\w_i, \vv)$ denotes the angle between $\w_i$ and $\vv$.
$\theta_{ij}\coloneqq\theta(\w_i, \w_j)$ denotes the angle between $\w_i$ and $\w_j$. 
Define 
\[\rr \coloneqq \sum_{i=1}^n \w_i -\vv,~~~\text{and}~~~
R: \mathbf{R}^d \to \mathbf{R}, R(\x)\coloneqq\sum_{j=1}^n[\bm{w}_j^\top \bm{x}]_+-[\bm{v}^\top \bm{x}]_+.\]
Then $L(\w)=\mathbb{E}_{\x}[\frac{1}{2}R^2(\x)]$. Define the length of the projection of $\w_i$ onto $\vv$ as \[h_i=\langle\w_i, \overline{\vv}\rangle.\] Lastly, define \[H\coloneqq\|\vv\|-\sum_{i\in[n]} h_i=\langle \overline{\vv}, -\rr\rangle.\]

\noindent\textbf{Closed Form Expressions of Loss and Gradient.}
When the input distribution is standard Gaussian, closed form expressions of $L(\w)$ and $\nabla L(\w)$ can be obtained \citep{safran18spurious}.  The complete form is deferred to Appendix \ref{Closed Form Expressions}. Here we only present the closed form of gradient as it is used extensively in our analysis: \citet{safran18spurious} showed that when $\w_i\neq \bm{0}, \forall i\in[n]$, the loss function is differentiable with gradient given by:
\begin{equation}\label{gradient}
    \begin{split}
    \nabla_i=\frac{1}{2}\left(\sum_j \w_j -\bm{v}\right)+\frac{1}{2\pi}\left[\left(\sum_{j\neq i}\|\w_j\|\sin\theta_{ij}-\|\vv\|\sin\theta_i\right)\overline{\w}_i-\sum_{j\neq i} \theta_{ij} \w_j +\theta_i\vv\right].
\end{split}
\end{equation}

\section{Proof Overview: Global Convergence}\label{Proof Sketch}
In this section we provide a proof sketch for Theorem \ref{informal global convergence}. Full proofs for all theorems and lemmas can be found in the Appendix. We start with the initialization.

\subsection{Initialization} 
We need the following conditions, which hold with high probability by random initialization.
\begin{lemma}\label{Initialization}
 Let $s_1 \coloneqq \frac{1}{2}\sigma \sqrt{d}, s_2\coloneqq 2\sigma \sqrt{d}$.    When $d=\Omega(\log(n/\delta))$, with probability at least $1-\delta$, the following properties hold at the initialization:
    \begin{equation}\label{Init norm condition}
        \forall i\in[n], s_1\leq \|\w_i(0)\|\leq s_2,~~~\text{and}~~~
        \frac{\pi}{3}\leq \theta_i(0)\leq \frac{2\pi}{3}.
    \end{equation}
\end{lemma}
Condition \eqref{Init norm condition} gives upper bound $s_2$ and lower bound $s_1$ for the norms of $\w_i(0)$, and states
$\theta_i$ will fall in the interval $[\frac{\pi}{3}, \frac{2\pi}{3}]$ initially. 
These are standard facts in high-dimensional probability. See Appendix \ref{proof of Initialization} for proof details.
The rest of our analysis will proceed \emph{deterministically.}

\subsection{Phase 1}
We present the main theorem of Phase 1, which starts at time $0$ and ends at time $T_1$.
\begin{theorem}[Phase 1]\label{p1}
Suppose the initial condition in Lemma \ref{Initialization} holds. For any $\epsilon_1 = O(1), (\epsilon_1>0
)$, there exists $C=O\left(\frac{\epsilon_1^2}{n}\right)$ such that for any $\sigma=O\left(C\epsilon_1^{48}d^{-1/2}\|\vv\|\right)$ and $\eta=O\left(\frac{nC\sigma\sqrt{d}}{\|\vv\|}\right)$, by setting $T_1\coloneqq \frac{C}{\eta},$
the following holds for $\forall 1\leq i\leq n, 0\leq t\leq T_1$: 
         \begin{equation}\label{main1}
             s_1\leq \|\w_i(t)\|\leq s_2+2\eta\|\vv\|t,
         \end{equation}
         \begin{equation}\label{main2}
   \text{and}~~~      \sin^2\left(\frac{\theta_i(t)}{2}\right)-\epsilon_1^2\leq\left(1+\frac{\eta t}{s_2/\|\vv\|}\right)^{-1/24}\left(\sin^2\left(\frac{\theta_i(0)}{2}\right)-\epsilon_1^2\right).
         \end{equation}
Consequently, at the end of Phase 1, we have
    \begin{equation}\label{maintheta_p1}
        \forall i\in [n], \theta_i(T_1)\leq 4\epsilon_1,
    \end{equation}
     \begin{equation}\label{main3}
 \text{and}~~~~        h_i(T_1)\leq 2h_j(T_1), \forall i,j\in[n].
     \end{equation}
\end{theorem} 
\eqref{main1} gives upper and lower bounds for $\|\w_i\|$. \eqref{main2} is used to bound the dynamics of $\theta_i$. \eqref{maintheta_p1} shows that $\theta_i$ is small at the end of Phase 1, so the student neurons are approximately aligned with the teacher neuron. \eqref{main3} states that the student neurons' projections on the teacher neuron are balanced.
Now we briefly describe our proof ideas.

\noindent\textbf{Proof of \eqref{main1}. } Proving the upper bound of $\|\w_i\|$ is straightforward, since triangle inequality implies an upper bound of gradient norm $\|\nabla_i\| = O(\|\vv\| + \sum_i\|\w_i\|)$, and the increasing rate of $\|\w_i\|$ is bounded by $\eta\|\nabla_i\|$. Note that we use $\|\w_i\|$ to upper bound $\|\nabla_i\|$, and use $\|\nabla_i\|$ to upper bound $\|\w_i\|$, so the argument can proceed inductively.

Given with the upper bound, we know that $\|\w_i\|=O(\eta\|\vv\|t)=O(\epsilon_1^2\|\vv\|/n)$ is a small term. Then the gradient \eqref{gradient} can be rewritten as: 
\begin{equation}\label{gd p1}
    \nabla_i = -\frac{1}{2\pi}(\|\vv\|\sin\theta_i\overline{\w}_i+(\pi-\theta_i)\vv)+O(\epsilon_1^2\|\vv\|^2).
\end{equation}
With \eqref{gd p1}, we prove the lower bound $\|\w_i\|\geq s_1$ by showing that $\|\w_i\|$ monotonically increases.

\noindent\textbf{Proof of \eqref{main2}.}
The condition \eqref{main2} aims to show that $\theta_i$ would decrease. Our intuition is clear: Since in each GD iteration, the update of $\w_i$ (the inverse of gradient \eqref{gd p1}) is approximately a linear combination of $\overline{\w}_i$ and $\vv$, the angle between $\w_i$ and $\vv$ is going to decrease.

However, there is a technical difficulty when converting the above intuition into a rigorous proof, which is caused by the small perturbation term $O(\epsilon_1^2\|\vv\|^2)$ in \eqref{gd p1}. When $\theta_i$ is large, showing $\theta_i$ would decrease is easy since this term is negligible. But when $\theta_i$ is too small, the effects of this perturbation term on the dynamics of $\theta_i$ is no longer negligible. As a result, we cannot directly show that $\theta_i$ decreases \emph{monotonically}.
Instead, we prove a weaker condition on the dynamics of $\theta_i$ and perform an algebraic trick (See \eqref{trick} \eqref{anand} in Appendix \ref{Phase 1 of Global Convergence}). Define $\chi_i(t)\coloneqq \sin^2\left(\frac{\theta_i(t)}{2}\right)$. We have 
\begin{equation}\label{anima}
    \begin{split}
        &\chi_i(t)-\chi_i(t+1)
\geq \frac{\eta\|\vv\|}{12\|\w_i(t+1)\|}\left(\chi_i(t)-\epsilon_1^2 \right)\\
\To  &\chi_i(t+1)-\epsilon_1^2
\leq \left(1-\frac{\eta\|\vv\|}{12\|\w_i(t+1)\|}\right)\left(\chi_i(t)-\epsilon_1^2 \right).
    \end{split}
\end{equation}
 Note that \eqref{anima} holds regardless of the sign of $\chi_i(t)-\epsilon_1^2$, hence both cases of $\theta_i$ being large and $\theta_i$ being small are gracefully handled. Therefore we can apply \eqref{anima} iteratively to get $\chi_i(t+1)-\epsilon_1^2\leq \Pi_{t'=1}^{t+1}\left(1-\frac{\eta\|\vv\|}{12\|\w_i(t')\|}\right)\left(\chi_i(0)-\epsilon_1^2 \right)$ (even if $\chi_i(t)-\epsilon_1^2$ might be negative for some $t$). This bound, combined with algebraic calculations, yields \eqref{main2}.

\noindent\textbf{Proof of  \eqref{maintheta_p1}. }
Applying \eqref{main2} with $t=T_1$ and some basic algebraic calculations yields \eqref{maintheta_p1}.

\noindent\textbf{Proof of \eqref{main3}.}
To prove \eqref{main3}, we divide Phase 1 into two intervals: $[0, T_1/50]$ and $[T_1/50, T_1]$. We first show that $\theta_i$ remains small in the second interval: $[T_1/50, T_1]$. Given with $\theta_i$ being small, nice properties of the gradient implies that $h_i$ monotonically increases, and its increasing rate approximately equals $\frac{\eta}{2}H$ (see \eqref{hdynamics}), which is identical for all $i$. Therefore, the increases of $h_i$ in the second interval:  $h_i(T_1) - h_i(T_1/50)$ are balanced. Then we show that $h_i(T_1/50)$ is small compared to $h_i(T_1) - h_i(T_1/50)$. These two properties together shows that $h_i(T_1)$ are balanced.

\subsection{Phase 2}
Our second phase starts at time $T_1 + 1$ and ends at time $T_2$. The main theorem is as follows.

\begin{theorem}[Phase 2]\label{p2}
Suppose the initial condition in Lemma \ref{Initialization} holds. For $\forall \epsilon_2=O(1)$, set $\epsilon_1=O\left(\epsilon_2^{6}n^{-1/2}\right)$ in Theorem \ref{p1}, $\eta =O\left(\frac{\epsilon_1^2\sigma^2d}{\|\vv\|^2}\right)$ and $T_2=T_1+\left\lceil\frac{1}{n\eta}\ln\left(\frac{1}{36\epsilon_2}\right)\right\rceil$, 
then $\forall T_1\leq t\leq T_2$,
\begin{equation}\label{m2-1}
   h_i(t)\leq 2h_j(t), \forall i,j,
\end{equation}
\begin{equation}\label{m2-2}
   \left(1-\frac{n\eta}{2}\right)^{t-T_1}\|\vv\|+6\epsilon_2\|\vv\|\geq H(t) \geq \frac{2}{3}\left(1-\frac{n\eta}{2}\right)^{t-T_1}\|\vv\|-6\epsilon_2\|\vv\|\geq 18\epsilon_2\|\vv\|,
\end{equation}
\begin{equation}\label{m2-3}
    \frac{2\|\vv\|}{n}\geq h_i(t)\geq \frac{s_1}{2}, \forall i.
\end{equation}
\begin{equation}\label{m2-4}
    \theta_i(t)\leq \epsilon_2, \forall i.
\end{equation}
\end{theorem}
\eqref{m2-1} is the continuation of \eqref{main1}, which shows that the projections $h_i$ remain balanced in Phase 2. \eqref{m2-2} bounds the dynamics of $H(t)$. It shows that $H(t)$ exponentially decreases and gives upper and lower bounds. \eqref{m2-3} gives upper and lower bounds for $h_i$. \eqref{m2-4} shows that $\theta_i$ remains upper bounded by a small term $\epsilon_2$ in Phase 2.
Below we prove \eqref{m2-1} \eqref{m2-2} \eqref{m2-3} \eqref{m2-4} together inductively.

\noindent\textbf{Proof of \eqref{m2-1}. }
Similar to \eqref{main3}, note that \eqref{m2-4} guarantees that $ \theta_i$ is small, so we still have that, for $\forall i$, $h_i$ monotonically increases with rate approximately $\frac{\eta}{2}H$. Therefore, $h_i$ will remain balanced.

\noindent\textbf{Proof of \eqref{m2-2}. } To understand why we need the bound \eqref{m2-2}, note that the gradient \eqref{gradient} has the following property:
\begin{equation}\label{gd p2}
    \nabla_i = \frac{1}{2}\rr + O((n\max_{i}\|\w_i\|+\|\vv\|)\max_i \theta_i).
\end{equation}
By \eqref{m2-3} and \eqref{m2-4}, $\max_i \theta_i\leq \epsilon_2
$, and $\max_{i}\|\w_i\|=O(\|\vv\|/n)$. So the second term in \eqref{gd p2} can be bounded as ${O}((n\max_{i}\|\w_i\|+\|\vv\|)\max_i \theta_i)\leq {O}(\epsilon_2\|\vv\|)$. When ${O}(\epsilon_2\|\vv\|)$ is much smaller than the first term $\rr/2$ in \eqref{gd p2}, we have $\nabla_i\approx\rr/2$. Consequently, $\rr$ and $H = \langle \overline{\vv}, -\rr\rangle$ will decrease with an exponential rate. But this will end when $\rr$ becomes no larger than ${O}(\epsilon_2\|\vv\|)$ and the approximation $\nabla_i\approx\rr/2$ no longer holds, and that is the end of Phase 2. 

So $H$ should decrease (with exponential rate) to a small value, and it also should not be too small to ensure that $\|\rr/2\|\gg{O}(\epsilon_2\|\vv\|)$ (since $H=\langle \overline{\vv}, -\rr\rangle$). So we need to use \eqref{m2-2} to simultaneously upper and lower bound $H$. With the above intuition, proving \eqref{m2-2} is straightforward as: $\nabla_i \approx \rr/2 \To H(t+1) \approx (1-n\eta/2) H(t)\approx\cdots\approx (1-n\eta/2)^{t-T_1} H(T_1)$. 

It is worth noting that, we also need to handle a perturbation term when bounding the dynamics of $H(t)$, and we used the same trick as in proving \eqref{main2}.

\noindent\textbf{Proof of \eqref{m2-3}. } 
The left inequality can be derived from \eqref{m2-1} and \eqref{m2-2}. The right inequality can be derived from the monotonicity of $h_i$, \eqref{m2-4} and \eqref{main1}.

\noindent\textbf{Proof of \eqref{m2-4}. } This is the most difficult part in Theorem \ref{p2}.
Recall that in Phase 1 we used the gradient approximation \eqref{gd p1} to bound $\theta_i$, but \eqref{gd p1} relies on $\|\w_i\|$ being a small term, which only holds in phase 1.
So this time we use a totally different method to bound $\theta_i$.

First we calculate the dynamics of $\cos\theta_i$ and get (see the proof in Appendix \ref{Phase 2 of Global Convergence} for details): $\cos(\theta_i(t+1))-\cos(\theta_i(t))=I_1+I_2$, where term $I_1\geq -\frac{\eta}{2}\sum_{j\neq i} \sin\theta_i(t)\sin(\theta_i(t)+\theta_j(t))\frac{\|\w_j(t)\|}{\|\w_i(t+1)\|}$, and  term $I_2$ is a small perturbation term. The next step is to establish the condition \eqref{m2-1}, then use it to bound the term $\frac{\|\w_j(t)\|}{\|\w_i(t+1)\|}$ in $I_1$. Consequently, we have 
\begin{equation}\label{hurry}
   \cos(\theta_i(t+1))-\cos(\theta_i(t))=I_1+I_2
\geq -{2\eta}\sum_{j\neq i} \sin\theta_i(t)(\sin\theta_i(t)+\sin\theta_j(t)) + I_2.
\end{equation}

However, this is still not enough to prove the bound. The lower bound of the dynamics of $\cos\theta_i$ in \eqref{hurry} depends on $\theta_j$ where $j\neq i$. Since $\theta_j$ might be much larger then $\theta_i$, the increasing rate of $\theta_i$ still cannot be upper-bounded.

To solve this problem, our key idea is to consider all $\theta_i$'s together. Define a potential function $V(t)\coloneqq\sum_i \sin^2\left(\theta_i(t)/2\right)$, then we can sum the bound in \eqref{hurry} over all $i$'s to get an upper bound for the increasing rate of $V$. Although the bound for $\theta_i$ depends on other $\theta_j$'s, the bound for $V$ only depends on $V$ itself. Consequently, the dynamics of the potential function $V$ can be upper bounded, which yields the final upper bound \eqref{m2-4}.

\subsection{Phase 3}
\begin{theorem}[Phase 3]\label{p3}
Suppose the initial condition in Lemma \ref{Initialization} holds. If we set $\epsilon_2=O(n^{-14})$ in Theorem \ref{p2}, $\eta=O\left(\frac{1}{n^2}\right)$, then $\forall T\in \mathbf{N}$ we have
    \begin{equation}\label{implicit regularization}
        \frac{4\|\vv\|}{n}\geq \|\w_i(T+T_2)\|\geq \frac{\|\vv\|}{4n}~~~\text{and}~~~    L(T+T_2)\leq O\left(\frac{n^4\|\vv\|^2}{\left(\eta T \right)^3}\right).
    \end{equation}

\end{theorem}
This is the desired $1/T^3$ convergence rate. 
Our analysis consists of two steps:\\
1. Prove a gradient lower bound $\|\nabla L(\w)\|\geq \text{poly}(n^{-1}, \|\vv\|^{-1})L^{2/3}(\w)$.\\
2. Prove that the loss function is smooth and Lipschitz on the gradient trajectory. \\
Given these two properties, the convergence can be  established via the standard analysis for GD.

\subsubsection{Step 1: Gradient Lower Bound. }
\begin{theorem}[Gradient Lower Bound]\label{gdbound}
    If for every student neuron we have $\frac{4\|\vv\|}{n}\geq \|\w_i\|\geq \frac{\|\vv\|}{4n}$, and 
    $L(\w)= O\left(\frac{\|\vv\|^2}{n^{14}}\right),$ then $\left\|\nabla_{\w}L(\w)\right\|\geq \Omega\left(\frac{{L^{2/3}(\w)}}{n^{2/3}\|\vv\|^{1/3}}\right).$

\end{theorem}
As stated in Section \ref{Technical Overview}, this theorem is an improved version of Theorem 3 in \cite{zhou2021local}, improving the dependency of $L$ from $\|\nabla_{\w} L(\w)\|\geq \Omega(L(\w))$ to $\|\nabla_{\w} L(\w)\|\geq \Omega(L^{2/3}(\w))$.
Below we introduce our idea of improving the bound. 

\begin{lemma}[Gradient Projection Bound]\label{gd_projection}
    Suppose $\w_1^*, \w_2^*, \ldots, \w_n^*$ is a global minimum of loss function $L$. Define $ \theta_{\max}:=\max_{i\in[n]} \theta_i$, then 
    \begin{equation}\label{up}
        \sum_{i=1}^n\left\langle\frac{\partial}{\partial \w_i}L(\w), \w_i-\w_i^*\right\rangle\geq 2L(\w)-O\left(\theta_{\max}^2\|\rr\|\cdot\|\vv\|\right).
    \end{equation}
\end{lemma}

Lemma \ref{gd_projection} uses the idea of ``descent direction'' from Lemma C.1 in \cite{zhou2021local}. The idea is to pick a global minimum $\w_1^*, \w_2^*, \ldots, \w_n^*$ and lower bound the projection of gradient on the direction $\w_i-\w_i^*$. 
Recall that \citet{zhou2021local} made artificial modifications of the network architecture for technical reasons, e.g., they used the absolute value activation $x\to|x|$ instead of ReLU. Therefore, their proof cannot be directly applied to our lemma. However, we show that their idea still works in our setting, and modified their proof to prove Lemma \ref{gd_projection} in Appendix \ref{Proofs for Gradient Lower Bound}.

With Lemma \ref{gd_projection} and several technical lemmas (Lemma \ref{theta}, \ref{r} in Appendix \ref{Proofs for Gradient Lower Bound}), it is easy show that the last term $O\left(\theta_{\max}^2\|\rr\|\cdot\|\vv\|\right)$ in \eqref{up} is small, so
 $\sum_{i=1}^n\left\langle\frac{\partial}{\partial \w_i}L(\w), \w_i-\w_i^*\right\rangle\geq L(\w).$ Then we need to upper bound $\|\w_i-\w_i^*\|$, and that is the step where we make the improvement. In \cite{zhou2021local}, they picked an \emph{arbitrary} global minimum $\{\w_i^*\}_{i=1}^n$ and treated the term $\|\w_i-\w_i^*\|$ as constantly large. Consequently, their gradient lower bound scale with $L^{-1}$, yielding a final convergence rate of $L(\w(T)) \leq  O(T^{-1})$. In contrast, our key observation is that we can pick a \emph{specific} global minimum $\{\w_i^*\}_{i=1}^n$ that depends on $\{\w_i\}_{i=1}^n$.
Specifically, we define 
\[\forall i\in[n], \w_i^*\coloneqq\frac{h_i}{\sum_j h_j}\vv.
\]
Then Lemma \ref{w-w*} shows that $\|\w_i-\w_i^*\|\leq O(L^{1/3}(\w))$ is a small term rather than a constant term. Finally, direct application of Cauchy inequality yields the improved bound Theorem \ref{gdbound}.

\subsubsection{Step 2: Smoothness and Lipschitzness}\label{Step2: Smoothness and Lipschitzness}

The aim of step 2 is to show the smoothness and Lipschitzness of $L$. However, one can see from \eqref{gradient} that $L$ is neither Lipschitz nor smooth. The problem of non-Lipschitzness is easy to address, since \eqref{gradient} implies that $\|\nabla L\|$ is upper bounded by $\|\w_i\|$, and $\|\w_i\|$ is upper bounded by $L(\w)$.
However, the non-smoothness property of $L$ is hard to handle. By the closed form expression of $\nabla^2 L$ (see \eqref{Hessian}), one can see that $\|\nabla^2 L\|$ scales with $\frac{\|\vv\|}{\|\w_i\|}$. Then $\|\nabla^2 L\|\to \infty$ as $\|\w_i\|\to 0$.

As stated in Section \ref{Technical Overview}, our idea of solving this problem is to show that GD implicitly regularizes $\w_i$ such that $\|\w_i\|$ is always lower and upper bounded, namely \eqref{implicit regularization} in Theorem \ref{p3}. This property ensures the smoothness of $L$ on GD trajectory (see Lemma \ref{smoothness} for details).

\noindent\textbf{Implicit Regularization of Student Neurons. } Next we describe our idea of proving the implicit regularization condition \eqref{implicit regularization}.
It is not hard to give $\|\w_i(T_2)\|$  lower and upper bounds (see Lemma \ref{p3init}). Therefore, we only need to show that the student neurons do not move very far in phase 3. In other words, we wish to bound $\sum_{t=T_2} ^T \eta\|\nabla L(\w(t))\|$ for $\forall T> T_2$. The intuition is very clear: in phase 3, the loss being small implies that the decrease of loss is small. Since the move of student neurons results in the decrease of loss, the change of $\|\w_i\|$ should also be small. However, the following subtlety emerges when constructing a rigorous proof.
 
 \noindent\textbf{The Importance of the Improved Gradient Lower Bound. }We want to emphasize that our improved gradient lower bound (Theorem \ref{gdbound}) is crucial for bounding the movement of student neurons $\sum_{t=T_2} ^T \eta\|\nabla L(\w(t))\|$.  
 There is an intuitive explanation for this: The weaker  bound $\|\nabla L(\w(t))\| \sim L(\w(t))$ implies the rate $L(\w(T))\sim \frac{1}{T}$ (i.e., the rate in \cite{zhou2021local}). Then $\|\nabla L(\w(t))\| \sim L(\w(t)) \sim \frac{1}{T}$ and $\sum_{t=T_2} ^T \eta\|\nabla L(\w(t))\|\sim \sum_{t=T_2} ^T \frac{1}{t}$. But the infinite sum $\sum_{t=T_2} ^\infty \frac{1}{t}$ diverges, so we cannot derive any meaningful bound.

 On the other hand, the improved gradient lower bound $\|\nabla L(\w(t))\| \sim L^{2/3}(\w(t))$ implies the convergence rate $L(\w(T))\sim\frac{1}{T^3} \To\|\nabla L(\w(t))\| \sim L^{2/3}(\w(t)) \sim \frac{1}{T^2} \To \sum_{t=T_2} ^T \eta\|\nabla L(\w(t))\|\sim \sum_{t=T_2} ^T \frac{1}{t^2}$, which is finite.
 See Lemma \ref{locality} for the rigorous argument.

\subsection{Main Theorem}
Now we are ready to state and prove the formal version of Theorem \ref{informal global convergence}.

\begin{theorem}[Global Convergence]\label{global_convergence}
    For $\forall \delta >0$, if $d=\Omega(\log(n/\delta))$, $\sigma= O\left(n^{-4226}d^{-1/2}\|\vv\|\right),$ $\eta =O\left(\frac{\sigma^2d}{n^{169}\|\vv\|^2}\right)$, then there exists $T_2 = O\left(\frac{\log n}{n\eta}\right)$ such that with probability at least $1-\delta$ over the initialization, for any $T\in\mathbf{N}$,
    $L(\w(T+T_2))\leq O\left(\frac{n^4\|\vv\|^2}{\left(\eta T \right)^3}\right). $
\end{theorem}

To combine three phases of our analysis together, the last step is to assign values to the parameters in Theorem \ref{p1}, \ref{p2}, \ref{p3} ($\epsilon_2, \epsilon_1, C, \sigma, \eta, T_1, T_2$) such that the previous phase satisfies the requirements of the next phase. 
For a complete list of the values, we refer the readers to Appendix \ref{valuation}.
With the parameter valuations in Appendix \ref{valuation}, combining the initialization condition (Lemma \ref{Initialization}) and three phases of our analysis (Theorem \ref{p1}, \ref{p2}, \ref{p3}) together proves Theorem \ref{global_convergence} immediately.

\begin{remark}
Careful readers might notice that, if there exists $i$ such that $\w_i = \bm{0}$, then $L$ is not differentiable and gradient descent is not well-defined. However, such a corner case has been naturally excluded in our previous analysis. (See Appendix \ref{Non-degeneracy of Student Neurons} for a detailed discussion.)
\end{remark}

\section{Proof Overview: Convergence Rate Lower Bound} \label{Proof Sketch: Convergence Rate Lower Bound}
In this section,  we provide a general overview for the convergence rate lower bound. Full proofs of all theorems can be found in Appendix \ref{Lower Bound of the Convergence Rate}.
We consider the gradient flow (gradient descent with infinitesimal step size):
\[\frac{\partial \w(t)}{\partial t}=-\frac{\partial L(\w(t))}{\partial \w}, \forall t \geq 0,\]
 while keeping other settings (network architecture, initialization scheme, etc.) unchanged.

 To understand why over-parameterization causes a significant change of the convergence rate, we first investigate several toy cases.
 
\subsection{Case study}\label{Case Study}

\noindent\textbf{Toy Case 1. } Set $n=2$, $\w_1(0) = \lambda_1(0)\vv + \lambda_2(0) \vv^{\bot}, \w_2(0) = \lambda_1(0)\vv - \lambda_2(0) \vv^{\bot}$, where $\lambda_1(0), \lambda_2(0)>0$, $\vv^{\bot}$ is a vector orthogonal with $\vv$ such that $\|\vv^\bot\|=\|\vv\|$. Then $\w_1(0)$ and $\w_2(0)$ are reflection symmetric with respect to $\vv$ (See Figure \ref{Toy case 1}). Consider gradient descent with step size $\eta$ initialized from $(\w_1(0), \w_2(0))$. It is easy to see that the symmetry of $\w_1$ and $\w_2$ is preserved in GD update, so for $ t = 0, 1, 2,\ldots$ there exists $\lambda_1(t), \lambda_2(t)$ such that  $\w_1(t) = \lambda_1(t)\vv + \lambda_2(t) \vv^{\bot}, \w_2(t) = \lambda_1(t)\vv - \lambda_2(t) \vv^{\bot}$. Since $\theta_1(t)=\theta_2(t), \forall t\in\mathbf{N}$, we denote $\theta \coloneqq \theta_1 = \theta_2$. 
Then gradient \eqref{gradient} has the form
\[\begin{split}
    \nabla_1=&\left(\lambda_1 -\frac{1}{2}\right)\vv+\frac{1}{2\pi}\left[\left(\|\w_2\|\sin(2\theta)-\|\vv\|\sin\theta\right)\overline{\w}_1-2\theta\w_2+\theta\vv\right]\\
    =&\left(\lambda_1-\frac{1}{2}+\frac{1}{2\pi}\left(\left(\sin(2\theta)-\frac{\|\vv\|}{\|\w_1\|}\sin\theta\right)\lambda_1-\theta(2\lambda_1-1)\right)\right)\vv\\
    &+\frac{1}{2\pi}\left(2\theta+\sin(2\theta)-\frac{\|\vv\|}{\|\w_1\|}\sin\theta\right)\lambda_2\vv^{\bot}\\
    =&\left(\lambda_1-\frac{1}{2}\right)\left(1-\frac{\theta}{\pi}+\frac{\sin(2\theta)}{\lambda_1}\right)\vv+\frac{1}{2\pi}\left(2\theta+\frac{\lambda_1 - 1/2}{\lambda_1}\sin(2\theta)\right)\lambda_2\vv^{\bot},
\end{split}\]
where the last equality is because $\sin(2\theta)-\frac{\|\vv\|}{\|\w_1\|}\sin\theta = \sin(2\theta)-\frac{\|\vv\|}{\lambda_1\|\vv\|/\cos\theta}\sin\theta =\frac{\sin(2\theta)}{\lambda_1}\left(\lambda_1-\frac{1}{2}\right)$. A similar expression can be computed for $\nabla_2$.

Then we can write out the dynamics of $\lambda_1$ and $\lambda_2$ as
\begin{equation}\label{toy1}
    \lambda_1(t+1) - \frac{1}{2} =  \left(\lambda_1(t)-\frac{1}{2}\right)\left(1-\eta\left(1-\frac{\theta(t)}{\pi}+\frac{\sin(2\theta(t))}{\lambda_1(t)}\right)\right),
\end{equation}
\begin{equation}\label{toy2}
    \lambda_2(t+1) = \lambda_2(t) \left(1 - \frac{\eta}{2\pi}\left(2\theta+\frac{\lambda_1 - 1/2}{\lambda_1}\sin(2\theta)\right)\right).
\end{equation}

Since $\theta = o(1)$, $\lambda_1$ is a constant term, $1-\frac{\theta(t)}{\pi}+\frac{\sin(2\theta(t))}{\lambda_1(t)}\approx 1$, then \eqref{toy1} implies \footnote{Here we use the $\approx$ sign to omit higher order terms.}  $ \lambda_1(t+1) - \frac{1}{2} \approx  \left(\lambda_1(t)-\frac{1}{2}\right)(1-\eta)$. This indicates that $\lambda_1$ converges to $\frac{1}{2}$ exponentially fast. So $\lambda_1 - 1/2 = o(1)\To 2\theta+\frac{\lambda_1 - 1/2}{\lambda_1}\sin(2\theta)\approx 2\theta\approx 2\tan\theta =2\frac{\lambda_2}{\lambda_1}\approx \lambda_2$. Then \eqref{toy2} can be rewritten as $\lambda_2(t+1) \approx \lambda_2(t) \left(1 - 
 \frac{\eta}{2\pi}\lambda_2(t)\right).$ This indicates that $\lambda_2$ converges to $0$ with rate $\lambda_2(t)\sim t^{-1}$. 
 
 Finally, we can compute the loss with \eqref{loss} as $L(\w)=\Theta\left((2\lambda_1 -1)^2+(\sin\theta-\theta\cos\theta)\right)\|\vv\|^2$. Since $(\sin\theta-\theta\cos\theta)\sim \theta^3 \sim \lambda_2^3\sim t^{-3}$, we know that the convergence rate is $L(\w(t))\sim t^{-3}$.

 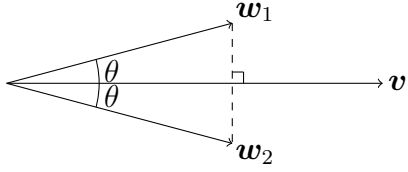
\begin{figure}[t]
\begin{center}
\begin{tikzpicture}
    \node  (1) at (5.2, 0) {$\vv$};
    \node (2) at (3.3, 0.95) {$\w_1$};
    \node (3) at (3.3, -0.95) {$\w_2$};
    \node (4) at (1.4, 0.17) {$\theta$};
    \node (5) at (1.4, -0.17) {$\theta$};

    \draw[->] (0,0) -- (5, 0);
    \draw[->] (0,0) -- (3, 0.8);
    \draw[->] (0,0) -- (3, -0.8);
    \draw[dashed] (3, 0.8) -- (3, -0.8);
    \draw[] (3.15,0) -- (3.15, 0.15);
    \draw[] (3.15,0.15) -- (3, 0.15);

    \path[clip] (3, 0.8) -- (3, -0.8) -- (0,0) --cycle;
    \node[circle, draw=black, minimum size=70pt] at (0,0) (circ) {};
\end{tikzpicture}
\end{center}
    \caption{Toy Case 1}
    \label{Toy case 1}
\end{figure}

\noindent\textbf{Toy Case 2. }
Let $n\geq 2$. We consider the case where all student neurons are parallel with the teacher neuron: $\w_1 = \lambda_1 \vv, \ldots, \w_n = \lambda_n\vv$, where $  \lambda_1,\ldots,\lambda_n \in\mathbf{R}^+$. Then the gradient \eqref{gradient} becomes $\nabla_i = \frac{1}{2}(\sum_j\w_j-\vv)$. One can easily see that $\sum_{i\in[n]} \lambda_i$ converges exponentially fast to $0$, which means that the convergence rate in this toy case is actually linear.

\noindent\textbf{Toy Case 3. }
Let $n\geq 2$. We consider the case where all student neurons are equal: $\w_1=\w_2=\ldots=\w_n$. Then the gradient \eqref{gradient} becomes $\nabla_i =\frac{1}{2}(n\w_i-\vv)+\frac{1}{2\pi}\left[-\|\vv\|\sin\theta_i\overline{\w}_i+\theta_i\vv\right]$. One can see that the gradient in this case is just $n$ times the gradient in the single student neuron case where the student neuron is $\w_i$ and the teacher neuron is $\vv/n$. So the training process is actually equivalent with learning one teacher neuron $\|\vv\|/n$ with one student neuron, with the step size $\eta$ being multiplied by a factor of $n$. So in this toy case, the loss also have linear convergence.

\subsection{Non-degeneracy}

Toy case 1 above already implies that the convergence rate given by Theorem \ref{global_convergence} is \emph{worst case optimal}. However, our ultimate goal is to prove an \emph{average case} lower bound for the convergence rate: Theorem \ref{lower bound main theorem}. One can see that there is a huge gap between the worst case optimality and the average case optimality. Proving the latter is much more difficult since toy case $2$ and $3$ exhibits fast-converging initialization points that break the $\Omega(T^{-3})$ lower bound.

Therefore, we need to utilize the property of random initialization to show our lower bound of $\Omega\left(1/T^3\right)$. Our idea is to show the lower bound holds as long as the initialization is non-degenerate.

To formalize the above idea, we first define a few important terms.
For $\forall i\in[n]$, define $\z_i:=\w_i-\langle\w_i, \overline{\vv}\rangle\overline{\vv}$ as the projection of $\w_i$ onto the orthogonal complement of $\vv$. Define $Z(t)=\sum_{1\leq i<j\leq n} \|\z_i(t)-\z_j(t)\|$.
Define $Q^+(t):=\{i\in[n]| \z_i(t)\neq \bm{0}\}$ as the index set containing all $i$ with $\z_i$ nonzero at time $t$. For $i,j\in Q^+$, define $\kappa_{ij}:=\theta(\z_i, \z_j)$ as the angle between $\z_i$ and $\z_j$. Define $\kappa_{\max}(t):=\max_{i,j\in Q^+(t)} \kappa_{ij}(t)$ as the maximum angle between $\z_i$ and $\z_j$.

\begin{definition}[Non-degeneracy]\label{non-degenerate}
When $n\geq 2$, we say the initialization is non-degenerate if the following two conditions are satisfied.
    (1) All $\z_i$'s are nonzero: $\forall i\in[n]$, $\z_i(0)\neq 0$.
    (2) $\z_i$'s are not parallel: $\kappa_{\max}(0)>0$.
\end{definition}
Since $\z_i$'s are initialized with a Gaussian distribution, the initialization is only degenerate on a set with Lebesgue measure zero, so the probability of the initialization being non-degenerate is $1$.
Now we are ready to state the formal version of Theorem \ref{Informal Lower Bound} whose proof is in Appendix \ref{Proof of Main Theorem: lower bound}.
\begin{theorem}[Convergence Rate Lower Bound] \label{lower bound main theorem}
Suppose the network is over-parameterized, i.e., $n\geq 2$. Consider gradient flow: $\frac{\partial \w(t)}{\partial t} = -\frac{\partial L(\w(t))}{\partial\w}.$ For $\forall \delta >0$, if the initialization is non-degenerate, $d=\Omega(\log(n/\delta))$, $\sigma= O\left(n^{-4226}d^{-1/2}\|\vv\|\right)$, then there exists $T_2 = O\left(\frac{\log n}{n}\right)$ such that with probability at least $1-\delta$, for $\forall t\geq T_2$  we have
\[L(\w(t))^{-1/3}\leq O\left(\frac{n^{17/3}}{\kappa_{\max}^2(0)\|\vv\|^{2/3}}\right)(t-T_2)+\gamma,\]
where $\gamma\in\mathbf{R}^+$ is a constant that does not depend on $t$.
\end{theorem}
\begin{remark}
The bound in Theorem \ref{lower bound main theorem} depends on $1/\kappa^{-2}_{\max}(0)$. Such a dependence is reasonable since we have shown that there would be counter-examples if $\kappa_{\max}(0)=0$ (See toy case 3 in Section \ref{Case Study}).
\end{remark}

\subsection{Proof Sketch}

Our key idea of proving Theorem \ref{lower bound main theorem} is to consider the potential function \[Z(t)=\sum_{i<j} \|\z_i(t)-\z_j(t)\| ~~~\text{where}~~~\z_i(t):=\w_i(t)-\langle\w_i(t), \overline{\vv}\rangle\overline{\vv}.\]
With $Z(t)$, our proof consists of three steps:\\
1. Show that with the non-degeneracy condition, $\kappa_{\max}(t)$ is lower bounded. (Lemma \ref{separation})\\
2. Show that when $\kappa_{\max}$ is lower bounded by a positive constant, $ \frac{\partial}{\partial t} Z(t)$ can be lower bounded by $Z^2(t)$ (See \eqref{z dynamics} in Appendix \ref{Proof of Main Theorem: lower bound}), so the convergence rate of $Z(t)$ is at most $Z(t) \sim t^{-1}$.\\
3. Use $Z(t)$ to lower bound $L(\w(t))$: $L(\w(t))\geq \Omega\left(\frac{Z^3(t)}{n^5\|\vv\|}\right).$ (\eqref{Z vs L} in Appendix \ref{Proof of Main Theorem: lower bound}).

\begin{remark}
    The potential function $Z(t)$ provides two implications.\\
$\bullet$ It explains why the convergence rate is different for $n=1$ and $n\geq 2$. For $n\geq 2$, our analysis implies that the slow convergence rate of $Z$ ($Z(t) \sim t^{-1}$) induces the slow convergence rate of $L(t)\sim t^{-3}$. When $n=1$, $Z$ is always zero, so $L$ converges with linear rate. Intuitively, this is because optimizing the difference between student neurons is hard, which is a phenomenon that only exists in the over-parameterized case.\\
$\bullet$ It explains why the convergence rates in the two counter-examples (toy case 2 and 3 in Section~\ref{Case Study}) are linear. In these two cases, the potential function $Z$ degenerates to $0$.
\end{remark}

We need several technical properties of the gradient flow trajectory.
The first one is the implicit regularization condition: \eqref{implicit regularization} in Theorem \ref{p3}, and  we use its gradient flow version (see Theorem \ref{p3flow} for details). We also need Corollary \ref{w>0 gd flow} and Lemma \ref{z>0} to exclude the corner cases when $\w_i = \bm{0}$ and $\z_i(t) = \bm{0}$, where $\kappa_{\max}$ is not well-defined. The proofs  are deferred to Appendix \ref{Implicit Regularization: Gradient Flow Version}.

Lower bounding $\kappa_{\max}$ is the most non-trivial step. We need to use the following lemma.

\begin{lemma}[Automatic Separation of $\z_i$]\label{max_kappa}
If there exists $i, j$ such that $\kappa_{ij}(t)=\kappa_{\max}(t) < \frac{\pi}{2}$, then $\cos\kappa_{ij}(t)$ is well-defined in an open neighborhood of $t$, differentiable at $t$, and
    \begin{equation}\label{shirley}
        \frac{\partial}{\partial t} \cos\kappa_{ij}(t)\leq -\frac{\pi-\theta_{ij}(t)}{\pi}(1-\cos\kappa_{ij}^2(t)).
    \end{equation}
\end{lemma}

Lemma \ref{max_kappa} states that, when the vectors $\z_i$ are too close in direction, gradient flow will automatically separate them, which immediately implies a lower bound of $\kappa_{\max}$ (See Theorem \ref{separation}). Its proof idea is also interesting: we can easily compute the dynamics of $\cos\kappa_{ij}$, which splits into two terms $I_1$ and $I_2$ (see \eqref{kappa dynamics} for them). $I_1$ is a simple term that can be handled easily, but the second term $I_2$ is very complicated and seems intractable. Our key observation is that, although $I_2$ is hard to bound for general $i,j$, it is always non-positive if we pick the pair of $i,j$ such that $\kappa_{ij}=\kappa_{\max}$, and that property implies Lemma \ref{max_kappa} via some routine computations.
\begin{remark}
We note that in toy case 3 in Appendix~\ref{Case Study}, all $\z_i$'s remain parallel and will not be separated. This is because the bound \eqref{shirley} in Lemma \ref{max_kappa} implies that the initial condition $\kappa_{ij} = 0,\forall i,j$ is unstable. To see this, consider the ordinary differential equation $\dot{x} = -\tilde{C}(1-x^2)$ where $\tilde{C}>0$ is a constant. The initial condition $x(0) = 1$ induces the solution $x(t) \equiv 1$, which corresponds to toy case 3. But this initial condition is unstable since any perturbation of $x(0)$ results in solution  $x(t)=\frac{1-\exp(2\tilde{C}t+c_0)}{1+\exp(2\tilde{C}t+c_0)}$, which implies an exponential increase of the perturbation, hence the separation of $\z_i$. 
\end{remark}

Given with a lower bound of $\kappa_{\max}(t)$, and the implicit regularization property in Theorem \ref{p3flow}, step 2 and step 3 can be proved with some geometric lemmas See Lemma \ref{partial Z>}, Lemma \ref{Z>} and the proof of Theorem \ref{lower bound main theorem} in Appendix \ref{Lower Bound of the Convergence Rate}. Combining three steps together finishes our proof.

%% file: Content/UpperBound.tex
\section{Closed Form Expressions for \texorpdfstring{$L$}{TEXT} and \texorpdfstring{$\nabla L$}{TEXT}}
\label{Closed Form Expressions}

In this section, we present closed forms of $L$ and $\nabla L$, as computed in  \cite{safran18spurious}.

\noindent\textbf{Closed Form of $L(\w)$.}
\[L(\w)=\frac{1}{2}\sum_{i, j=1}^n \Upsilon(\bm{w}_i, \bm{w}_j)-\sum_{i=1}^n \Upsilon(\bm{w}_i, \bm{v})+\frac{1}{2}\Upsilon(\bm{v}, \bm{v}),\]
where
\[\begin{split}
    \Upsilon(\bm{w}, \bm{v})&=\mathbb{E}_{\bm{x}\sim\mathcal{N}(\bm{0}, I)})\left[\left[\bm{w}^\top \bm{x}\right]_+\left[\bm{v}^\top \bm{x}\right]_+\right]\\
    &=\frac{1}{2\pi}\|\bm{w}\|~\|\bm{v}\|(\sin(\theta_{\bm{w}, \bm{v}})+(\pi-\theta_{\bm{w}, \bm{v}})\cos(\theta_{\bm{w}, \bm{v}})).
\end{split}\]

Rearranging terms yields
\begin{equation}\label{loss}
    L(\w)=\frac{1}{4}\|\sum_i \w_i - \bm{v}\|^2+\frac{1}{2\pi}\left[\sum_{i< j}(\sin\theta_{ij}-\theta_{ij}\cos\theta_{ij})\|\w_i\|~\|\w_j\|-\sum_i(\sin\theta_i-\theta_i\cos\theta_i)\|\w_i\|~\|\bm{v}\|\right].
\end{equation}

\noindent\textbf{Closed Form of $\nabla L(\w)$.}
When $\w_i\neq \bm{0}, \forall i\in[n]$, \citet{safran18spurious} showed that the loss function is differentiable and the gradient is given by
\begin{equation*}
    \begin{split}
    \nabla_i&=\mathbb{E}_{\bm{x}\sim\mathcal{N}(\bm{0}, I)}\left[\left(\sum_{j=1}^n\left[\bm{w}_j^\top \bm{x}\right]_+-\left[\bm{v}^\top \bm{x}\right]_+\right)\mathbb{1}\left\{\w_i^\top \x\geq 0\right\}\x\right]\\    
    &=\frac{1}{2}\left(\sum_j \w_j -\bm{v}\right)+\frac{1}{2\pi}\left[\left(\sum_{j\neq i}\|\w_j\|\sin\theta_{ij}-\|\vv\|\sin\theta_i\right)\overline{\w}_i-\sum_{j\neq i} \theta_{ij} \w_j +\theta_i\vv\right].
\end{split}
\end{equation*}

\section{Global Convergence: phase 1}\label{Phase 1 of Global Convergence}
\begin{reptheorem}{p1}
Suppose the initial condition in Lemma \ref{Initialization} holds. For any $\epsilon_1 = O(1), (\epsilon_1>0
)$, there exists $C=O\left(\frac{\epsilon_1^2}{n}\right)$ such that for any $\sigma=O\left(C\epsilon_1^{48}d^{-1/2}\|\vv\|\right)$ and $\eta=O\left(\frac{nC\sigma\sqrt{d}}{\|\vv\|}\right)$, by setting $T_1\coloneqq \frac{C}{\eta},$\footnote{Here we set $\eta$ such that $T_1=C/\eta\in\mathbf{N}$.} the following holds for $\forall 1\leq i\leq n, 0\leq t\leq T_1$: 
         \begin{equation}\label{1}
             s_1\leq \|\w_i(t)\|\leq s_2+2\eta\|\vv\|t, 
         \end{equation}
         \begin{equation}\label{2}
             \sin^2\left(\frac{\theta_i(t)}{2}\right)-\epsilon_1^2\leq\left(1+\frac{\eta t}{s_2/\|\vv\|}\right)^{-1/24}\left(\sin^2\left(\frac{\theta_i(0)}{2}\right)-\epsilon_1^2\right).
         \end{equation}
Consequently, at the end of Phase 1, we have
    \begin{equation}\label{theta_p1}
        \forall i\in [n], \theta_i(T_1)\leq 4\epsilon_1,
    \end{equation}
and
     \begin{equation}\label{3}
         h_i(T_1)\leq 2h_j(T_1), \forall i,j\in[n].
     \end{equation}
\end{reptheorem}

\begin{proof}
By Lemma \ref{Initialization}, \eqref{1} and \eqref{2} holds for $t=0$, and we have $s_1\leq \|\w_i(0)\|\leq s_2, \forall i$.

Now we show with induction that \eqref{1} and \eqref{2} holds for $\forall t\leq T_1$.

    For $t<T_1$, assume \eqref{1} and \eqref{2} holds for $0, 1, \ldots, t$, we prove the case of $t+1$.
    
    First note that  \eqref{2} holds for $0, 1, \ldots, t$ implies that $\forall t'\leq t$, $\sin^2(\theta_i(t')/2)\leq \max\{\sin^2(\theta_i(0)/2),\epsilon_1^2\}\leq\sin^2(\pi/3)\To \theta_i(t')\leq 2\pi/3$.
    
 \noindent   \textbf{Proof of the right inequality of \eqref{1}.} Consider $\forall 0\leq t'\leq t$, note that $\|\w_i(t')\|\geq s_1>0, \forall i$ implies that, for any $i,j$, the gradient $\nabla_i(t')$ and the angles $\theta_i(t'), \theta_{ij}(t')$ are  well-defined.
    
    Note that $s_2= 2\sigma \sqrt{d}  =O\left((\eta T_1)\epsilon_1^{48}\|\vv\|\right)\leq \eta T_1\|\vv\|$, so $\forall 0\leq t'\leq t$ we have 
    \begin{equation}\label{|wi|}
        \|\w_i(t')\|\leq s_2+2\eta \|\vv\|t'\leq s_2+2\eta \|\vv\|T_1\leq 3C \|\vv\|=O\left({\epsilon_1^2/n}\right)\|\vv\|\leq \frac{\|\vv\|}{3n}.
    \end{equation}
    
    By triangle inequality, for $\forall i\in[n], 0\leq t'\leq t$, 
    \begin{equation}\label{gdnorm}
        \begin{split}
        \|\nabla_i(t')\|&\leq\frac{1}{2}\left(\sum_j \left\|\w_j(t') \right\|+\|\vv\|\right)+\frac{1}{2\pi}\left[\sum_{j\neq i}\|\w_j(t')\|+\|\vv\|+\sum_{j\neq i} \pi \|\w_j(t')\| +\pi\|\vv\|\right]\\
        &\leq\frac{1}{2}\left(\frac{1}{3}+1\right)\|\vv\|+\frac{1}{2\pi}\left(\frac{1}{3}+1+\frac{\pi}{3}+\pi\right)\|\vv\|\leq 2\|\vv\|.
    \end{split}
    \end{equation}

    Then $\|\w_i(t+1)\|$ can be upper-bounded as \[\|\w_i(t+1)\|=\|\w_i(0)-\sum_{t'=0}^t \eta\nabla_i(t')\|\leq \|\w_i(0)\|+\sum_{t'=0}^t \eta\|\nabla_i(t')\|\leq s_2+2\eta(t+1)\|\vv\|.\]
    
\noindent \textbf{Proof of the left inequality of \eqref{1}.} Next we show that $\|\w_i(t+1)\|\geq\|\w_i(t)\|\geq s_1$.
Note that 
\[\|\w_i(t+1)\|^2-\|\w_i(t)\|^2=\|\w_i(t)-\eta\nabla_i(t)\|^2-\|\w_i(t)\|^2=-2\eta\langle\w_i(t),\nabla_i(t)\rangle+\eta^2\|\nabla_i(t)\|^2,\]
    so to show $\|\w_i(t+1)\|\geq \|\w_i(t)\|$, we only need to prove that $\langle\overline{\w}_i(t),\nabla_i(t)\rangle<0$ (note that by the induction hypothesis we have $\|\w_i(t)\|>0$, therefore $\overline{\w}_i(t)$ is well-defined):
    \[\begin{split}
        &\langle\overline{\w}_i(t),\nabla_i(t)\rangle\\
        =&-\frac{(\pi-\theta_i(t))\left\langle\overline{\w}_i(t),\vv\right\rangle+\left\langle\overline{\w}_i(t),\|\vv\|\sin\theta_i(t)\overline{\w}_i(t)\right\rangle}{2\pi}\\
        &+\sum_j \frac{\left\langle\overline{\w}_i(t),(\pi-\theta_{ij}(t))\w_j(t)\right\rangle}{2\pi}+\frac{\left\langle\overline{\w}_i(t),\left(\sum_{j\neq i}\|\w_j(t)\|\sin\theta_{ij}(t)\right)\overline{\w}_i(t)\right\rangle}{2\pi}\\
        =&-\frac{(\pi-\theta_i(t))\cos\theta_i(t)+\sin\theta_i(t)}{2\pi}\|\vv\|
        +\sum_j\frac{(\pi-\theta_{ij}(t))\cos\theta_{ij}(t)+\sin\theta_{ij}(t)}{2\pi}\|\w_j(t)\|\\
        \overset{\eqref{|wi|}}{\leq}&-\frac{(\pi-\theta_i(t))\cos\theta_i(t)+\sin\theta_i(t)}{2\pi}\|\vv\|
        +O\left(\epsilon_1^2\right)\|\vv\|\\
        \leq& -\frac{1/12}{2\pi}\|\vv\|+O\left(\epsilon_1^2\right)\|\vv\|\\
        <&0.
    \end{split}\]

The reason for the second to last inequality is that, it is easy to verify by taking derivatives that the expression $(\pi-\theta)\cos\theta+\sin\theta$ monotonically decreases on the interval $[0, \pi]$, and the induction hypothesis implies $\theta_i(t)\leq 2\pi/3$, therefore $(\pi-\theta_i(t))\cos\theta_i(t)+\sin\theta_i(t)\geq (\pi-2\pi/3)\cos(2\pi/3)+\sin(2\pi/3)>1/12.$

Then we have $\|\w_i(t+1)\|\geq\|\w_i(t)\|\geq s_1$.
    
\noindent    \textbf{Proof of \eqref{2}.}  First we calculate the dynamics of $\cos\theta_i$.
    \begin{equation}
        \begin{split}\label{cos-theta}
        &\cos(\theta_i(t+1))-\cos(\theta_i(t))\\
        =&\left\langle\overline{\w}_i(t+1),\overline{\vv}\right\rangle-\left\langle\overline{\w}_i(t),\overline{\vv}\right\rangle\\
        =&\frac{\|\w_i(t)\|\left\langle{\w}_i(t+1),\overline{\vv}\right\rangle-\|\w_i(t+1)\|\left\langle{\w}_i(t),\overline{\vv}\right\rangle}{\|\w_i(t+1)\|\cdot\|\w_i(t)\|}\\
        =&\frac{\left\langle{\w}_i(t),\overline{\vv}\right\rangle\left(\|\w_i(t)\|-\|\w_i(t+1)\|\right)-\eta\|\w_i(t)\|\left\langle\nabla_i(t),\overline{\vv}\right\rangle}{\|\w_i(t+1)\|\cdot\|\w_i(t)\|}\\
        =&\frac{\left\langle{\w}_i(t),\overline{\vv}\right\rangle\frac{\|\w_i(t)\|^2-\|\w_i(t)-\eta\nabla_i(t)\|^2}{\|\w_i(t)\|+\|\w_i(t+1)\|}-\eta\|\w_i(t)\|\left\langle\nabla_i(t),\overline{\vv}\right\rangle}{\|\w_i(t+1)\|\cdot\|\w_i(t)\|}\\
        =&\frac{1}{\|\w_i(t+1)\|}
        \left[\left\langle\overline{\w}_i(t),\overline{\vv}\right\rangle
        \frac{2\eta\left\langle\w_i(t), \nabla_i(t)\right\rangle-\eta^2\|\nabla_i(t)\|^2}{\|\w_i(t)\|+\|\w_i(t+1)\|}
        -\eta\left\langle\nabla_i(t),\overline{\vv}\right\rangle\right]\\
        =&\frac{1}{\|\w_i(t+1)\|}
        \biggl[\eta\left\langle\overline{\w}_i(t),\overline{\vv}\right\rangle
        \left\langle\overline{\w}_i(t), \nabla_i(t)\right\rangle
        +\left\langle\overline{\w}_i(t),\overline{\vv}\right\rangle
        \left(\frac{2\eta\left\langle\w_i(t), \nabla_i(t)\right\rangle}{\|\w_i(t)\|+\|\w_i(t+1)\|}-\frac{2\eta\left\langle\w_i(t), \nabla_i(t)\right\rangle}{2\|\w_i(t)\|}\right)\\
        &-\eta^2\left\langle\overline{\w}_i(t),\overline{\vv}\right\rangle
        \frac{\|\nabla_i(t)\|^2}{\|\w_i(t)\|+\|\w_i(t+1)\|}
        -\eta\left\langle\nabla_i(t),\overline{\vv}\right\rangle\biggr]\\
        =&\underbrace{\frac{\eta}{\|\w_i(t+1)\|}
        \left\langle\left\langle\overline{\w}_i(t),\overline{\vv}\right\rangle\overline{\w}_i(t)-\overline{\vv}, \nabla_i(t)\right\rangle}_{I_1}\\
        &+\underbrace{\frac{\eta\left\langle\overline{\w}_i(t),\overline{\vv}\right\rangle}{\|\w_i(t+1)\|}\biggl[
        \frac{\left\langle\overline{\w}_i(t), \nabla_i(t)\right\rangle\left(\|\w_i(t)\|-\|\w_i(t+1)\|\right)}{\|\w_i(t)\|+\|\w_i(t+1)\|}
        -\eta
        \frac{\|\nabla_i(t)\|^2}{\|\w_i(t)\|+\|\w_i(t+1)\|}\biggr]}_{I_2}.
    \end{split}
    \end{equation}

For the first term $I_1$, note that the vector $\left\langle\overline{\w}_i(t),\overline{\vv}\right\rangle\overline{\w}_i(t)-\overline{\vv}$ is orthogonal with $\w_i$, therefore,

\begin{equation}\label{T1}
    \begin{split}
    I_1&=\frac{\eta}{\|\w_i(t+1)\|}
\left\langle\left\langle\overline{\w}_i(t),\overline{\vv}\right\rangle\overline{\w}_i(t)-\overline{\vv}, \frac{1}{2\pi}\left[\sum_{j\neq i}(\pi-\theta_{ij}(t))\w_j(t)-(\pi-\theta_i(t))\vv\right]\right\rangle\\
&=\frac{\eta}{2\pi\|\w_i(t+1)\|}\left[(\pi-\theta_i(t))\sin^2\theta_i(t)\|\vv\|-\sum_{j\neq i}(\pi-\theta_{ij}(t))\left(\cos\theta_j(t)-\cos\theta_i(t)\cos\theta_{ij}(t)\right)\|\w_j(t)\|\right]\\
&\geq \frac{\eta}{2\pi\|\w_i(t+1)\|}\left[(\pi-\theta_i(t))\sin^2\theta_i(t)\|\vv\|-n\pi\cdot 2\|\w_j(t)\|\right]\\
&\overset{\eqref{|wi|}}{\geq} \frac{\eta}{2\pi\|\w_i(t+1)\|}\left[(\pi-\theta_i(t))\sin^2\theta_i(t)\|\vv\|-n\pi\cdot 2O(C)\|\vv\|\right]\\
&\geq \frac{\eta\|\vv\|}{2\pi\|\w_i(t+1)\|}\left[\frac{\pi}{3}\sin^2\theta_i(t)-O\left(nC\right)\right],
\end{split}
\end{equation}
where the last inequality is because $\theta_i(t)\leq\pi/3$.

The second term $I_2$ is a small perturbation term, which can be lower bounded as:
\begin{equation}\label{T2bound}
    \begin{split}
    &I_2\geq -\frac{\eta}{\|\w_i(t+1)\|}\left[\frac{\|\nabla_i(t)\|\cdot\|\eta\nabla_i(t)\|}{2s_1}+\eta\frac{\|\nabla_i(t)\|^2}{2s_1}\right]\\
    &=-\frac{\eta^2}{s_1\|\w_i(t+1)\|}\|\nabla_i(t)\|^2\overset{\eqref{gdnorm}}{\geq}-\frac{4\eta^2}{s_1\|\w_i(t+1)\|}\|\vv\|^2.
\end{split}
\end{equation}
Combining both terms together, we get
\[\begin{split}
    \cos(\theta_i(t+1))-\cos(\theta_i(t))=I_1+I_2
&\geq\frac{\eta\|\vv\|}{2\pi\|\w_i(t+1)\|}\left[\frac{\pi}{3}\sin^2\theta_i(t)-O\left(nC\right)-8\pi\eta\frac{\|\vv\|}{s_1}\right]\\
&\geq\frac{\eta\|\vv\|}{6\|\w_i(t+1)\|}\left[\sin^2\theta_i(t)-O\left(nC\right)\right],
\end{split}\]
where the last inequality is because $\eta=O\left(\frac{nC\sigma\sqrt{d}}{\|\vv\|}\right) \To 8\pi\eta\frac{\|\vv\|}{s_1}=O\left(nC\right)$.

Therefore, we have \begin{equation}\label{trick}
    \begin{split}
    &\sin^2\left(\frac{\theta_i(t)}{2}\right)-\sin^2\left(\frac{\theta_i(t+1)}{2}\right)
    =\frac{\cos(\theta_i(t+1))-\cos(\theta_i(t))}{2}\\
&\geq\frac{\eta\|\vv\|}{12\|\w_i(t+1)\|}\left[\sin^2\theta_i(t)-O\left(nC\right)\right]
\geq \frac{\eta\|\vv\|}{12\|\w_i(t+1)\|}\left[\sin^2\left(\frac{\theta_i(t)}{2}\right)-\epsilon_1^2 \right],
\end{split}
\end{equation} 
where the last inequality is because $\cos(\theta_{i}(t)/2)\geq \cos(\pi/3)=1/2\To \sin\theta_i(t)=2\sin(\theta_i(t)/2)\cos(\theta_{i}(t)/2)\geq \sin(\theta_i(t)/2)$, and $C= O(\epsilon_1^2/n)\To O(nC)\leq \epsilon_1^2.$

Then we have
\begin{equation*}
    \begin{split}
    \sin^2\left(\frac{\theta_i(t+1)}{2}\right)-\epsilon_1^2
    &\leq \sin^2\left(\frac{\theta_i(t)}{2}\right)-\frac{\eta\|\vv\|}{12\|\w_i(t+1)\|}\left[\sin^2\left(\frac{\theta_i(t)}{2}\right)-\epsilon_1^2\right]-\epsilon_1^2\\
    &=\left(1-\frac{\eta\|\vv\|}{12\|\w_i(t+1)\|}\right)\left(\sin^2\left(\frac{\theta_i(t)}{2}\right)-\epsilon_1^2\right)\\
    &\leq \left(1-\frac{\eta}{12\left(s_2/\|\vv\|+2\eta (t+1)\right)}\right)\left(\sin^2\left(\frac{\theta_i(t)}{2}\right)-\epsilon_1^2\right).
\end{split}
\end{equation*}

For the same reason, for any $t'\in\{0,1,\ldots, t\}$ we have
\begin{equation}\label{anand}
    \sin^2\left(\frac{\theta_i(t'+1)}{2}\right)-\epsilon_1^2
    \leq \left(1-\frac{\eta}{12\left(s_2/\|\vv\|+2\eta (t'+1)\right)}\right)\left(\sin^2\left(\frac{\theta_i(t')}{2}\right)-\epsilon_1^2\right).
\end{equation}

 $\sin^2\left(\frac{\theta_i(t')}{2}\right)-\epsilon_1^2$ can both be positive or negative, but \eqref{anand} always holds regardless of its sign. Since $1-\frac{\eta}{12\left(s_2/\|\vv\|+2\eta (t'+1)\right)}$ is always positive, and multiplying both sides of an inequality by a positive number does not change the direction of the inequality, we can iteratively apply \eqref{anand} and get
\[\begin{split}
    \sin^2\left(\frac{\theta_i(t+1)}{2}\right)-\epsilon_1^2
    \leq& \prod_{u=1}^{t+1}\left(1-\frac{\eta}{12\left(s_2/\|\vv\|+2\eta u\right)}\right)\left(\sin^2\left(\frac{\theta_i(0)}{2}\right)-\epsilon_1^2\right)\\
    \leq& \prod_{u=1}^{t+1}\exp\left(-\frac{\eta}{12\left(s_2/\|\vv\|+2\eta u\right)}\right)\left(\sin^2\left(\frac{\theta_i(0)}{2}\right)-\epsilon_1^2\right)\\
     \leq& \exp\left(\int_{u=1}^{t+2}-\frac{\eta}{12\left(s_2/\|\vv\|+2\eta u\right)}du\right)\left(\sin^2\left(\frac{\theta_i(0)}{2}\right)-\epsilon_1^2\right)\\
     =&\exp\left(-\frac{1}{24}\ln\left(\frac{s_2+(t+2)2\eta\|\vv\|}{s_2+2\eta\|\vv\|}\right)\right)\left(\sin^2\left(\frac{\theta_i(0)}{2}\right)-\epsilon_1^2\right)\\
     \leq& \left(1+\frac{\eta (t+1)}{s_2/\|\vv\|}\right)^{-1/24}\left(\sin^2\left(\frac{\theta_i(0)}{2}\right)-\epsilon_1^2\right),
\end{split}\]
where the last inequality is because $2\eta\|\vv\|\leq s_1\leq s_2$. (Note that, by Lemma \ref{Initialization}, $\sin^2\left(\frac{\theta_i(0)}{2}\right)-\epsilon_1^2$ is always positive.)

\noindent \textbf{Proof of \eqref{theta_p1}.}
W.L.O.G., suppose $T_1/50\in\mathbb{N}$.
By \eqref{2}, for $\forall t\in[T_1/50, T_1]$ we have that
\begin{equation}\label{V(t1)}
    \sin^2\left(\frac{\theta_i(t)}{2}\right)-\epsilon_1^2\leq\left(1+\frac{\eta t}{s_2/\|\vv\|}\right)^{-1/24}\left(\sin^2\left(\frac{\theta_i(0)}{2}\right)-\epsilon_1^2\right)\leq \left(\frac{\eta T_1/50}{O\left((\eta T_1)\epsilon_1^{48}\right)}\right)^{-1/24}\leq \epsilon_1^2.
\end{equation}

Since $\epsilon_1=O(1)$ is a sufficiently small constant, we have
\[\forall t\in[T_1/50, T_1], \sin\left(\frac{\theta_i(t)}{2}\right)\leq \sqrt{2}\epsilon_1
\To \forall t\in[T_1/50, T_1], \theta_i(t)\leq 4\epsilon_1.\]
This implies \eqref{theta_p1} immediately.

\noindent \textbf{Proof of \eqref{3}.}
Consider $\forall t\in[T_1/50, T_1]$. The dynamics of $h_i$ is given by 
\begin{equation}\label{hdynamics}
    \begin{split}
    &h_i(t+1)-h_i(t)=-\eta\langle\nabla_i(t),\overline{\vv}\rangle\\
    =&\underbrace{\frac{\eta}{2}{\left(\|\bm{v}\|-\sum_j h_j(t) \right)}}_{\frac{\eta}{2}H(t)}\\
    &-\underbrace{\frac{\eta}{2\pi}\left[\left(\sum_{j\neq i}\|\w_j(t)\|\sin\theta_{ij}(t)-\|\vv\|\sin\theta_i(t)\right)\cos\theta_i(t)-\sum_{j\neq i} \theta_{ij}(t) h_j(t)     +\theta_i(t)\|\vv\|\right]}_{Q_i(t)}.
\end{split}
\end{equation}
The first term is just $\frac{\eta}{2}H(t)$. Denote the second term with $Q_i(t)$. Then $h_i(t+1)=h_i(t)+\frac{\eta}{2}H(t)-Q_i(t)$.

$H(t)$ can be lower bounded as
\begin{equation}\label{T3}
    H(t)=\|\bm{v}\|-\sum_j h_j(t)\geq \|\bm{v}\|-\sum_j \|\w_j(t)\| \overset{\eqref{|wi|}}{\geq} \|\vv\|-n\cdot \frac{\|\vv\|}{3n}=\frac{2\|\vv\|}{3}.
\end{equation}

On the other hand, the second term $Q_i(t)$ is a small perturbation term, whose norm can be upper bounded by
\begin{equation}\label{Qnorm}
    \begin{split}
    |Q_i(t)|
    \leq &\frac{\eta}{2\pi}\left[n\cdot\frac{\|\vv\|}{3n}\theta_{ij}(t)+\|\vv\|\theta_i(t)+n\cdot\frac{\|\vv\|}{3n}\theta_{ij}(t)+\theta_i(t)\|\vv\|\right]\\
    \leq& \frac{\eta}{2\pi}\left[\frac{\|\vv\|}{3}8\epsilon_1+\|\vv\|4\epsilon_1+\frac{\|\vv\|}{3}8\epsilon_1+4\epsilon_1\|\vv\|\right]\\
    \leq&3\eta\|\vv\|\epsilon_1
    \leq9\epsilon_1\frac{\eta}{2}H(t)\leq 0.1 \frac{\eta}{2}H(t),
\end{split}
\end{equation}
where the second inequality is because $\theta_{ij}(t)\leq \theta_i(t)+\theta_j(t)\leq 8\epsilon_1$.

Therefore
\[0.3\eta\|\vv\|\leq 0.9\frac{\eta}{2}H(t)\leq h_i(t+1)-h_i(t)=\frac{\eta}{2}H(t)-Q_i(t)\leq 1.1\frac{\eta}{2}H(t), \forall t\in[T_1/50,T_1].\]
Then we have the following bound, which shows that, $\forall i$, $h_i(T_1)-h_i(T_1/50)$ approximately equals to $\sum_{t=T_1/50}^{T_1-1}\frac{\eta}{2}H(t)$:
\begin{equation}\label{u1}
    0.9\left(\sum_{t=T_1/50}^{T_1-1}\frac{\eta}{2}H(t)\right)
\leq h_i(T_1)-h_i\left(T_1/50\right)=\sum_{t=T_1/50}^{T_1-1}\left(h_i(t+1)-h_i(t)\right)
\leq 1.1 \left(\sum_{t=T_1/50}^{T_1-1}\frac{\eta}{2}H(t)\right).
\end{equation}

The next bounds shows that, $\forall i$, $|h_i(T_1/50)|$ is small comparing to $\sum_{t=T_1/50}^{T_1-1}\frac{\eta}{2}H(t)$:
\begin{equation}\label{u2}
    |h_i(T_1/50)|\leq \|\w_i(T_1/50)\|\overset{\eqref{1}}{\leq} s_2+2\eta\|v\|\frac{T_1}{50}\leq \frac{1}{20}\eta\|\vv\|T_1
\overset{\eqref{T3}}{\leq} 0.2 \left(\sum_{t=T_1/50}^{T_1-1}\frac{\eta}{2}H(t)\right).
\end{equation}

\eqref{u1} and \eqref{u2} jointly yields
\[0<0.7\left(\sum_{t=T_1/50}^{T_1-1}\frac{\eta}{2}H(t)\right)
\leq h_i(T_1)
\leq 1.3 \left(\sum_{t=T_1/50}^{T_1-1}\frac{\eta}{2}H(t)\right), \forall i,\]
 which implies \eqref{3} immediately.
    
\end{proof}

\section{Global Convergence: phase 2} \label{Phase 2 of Global Convergence}

\begin{reptheorem}{p2}
Suppose the initial condition in Lemma \ref{Initialization} holds. For $\forall \epsilon_2=O(1)$, set $\epsilon_1=O\left(\epsilon_2^{6}n^{-1/2}\right)$ in Theorem \ref{p1}, $\eta =O\left(\frac{\epsilon_1^2\sigma^2d}{\|\vv\|^2}\right)$ and $T_2=T_1+\left\lceil\frac{1}{n\eta}\ln\left(\frac{1}{36\epsilon_2}\right)\right\rceil$, 
then $\forall T_1\leq t\leq T_2$,
\begin{equation}\label{2-1}
   h_i(t)\leq 2h_j(t), \forall i,j,
\end{equation}
\begin{equation}\label{2-2}
   \left(1-\frac{n\eta}{2}\right)^{t-T_1}\|\vv\|+6\epsilon_2\|\vv\|\geq H(t) \geq \frac{2}{3}\left(1-\frac{n\eta}{2}\right)^{t-T_1}\|\vv\|-6\epsilon_2\|\vv\|\geq 18\epsilon_2\|\vv\|,
\end{equation}
\begin{equation}\label{2-3}
    \frac{2\|\vv\|}{n}\geq h_i(t)\geq \frac{s_1}{2}, \forall i.
\end{equation}
\begin{equation}\label{2-4}
    \theta_i(t)\leq \epsilon_2, \forall i.
\end{equation}
\end{reptheorem}

\begin{proof}
    
    We prove \eqref{2-1}, \eqref{2-2}, \eqref{2-3} and \eqref{2-4} together inductively. First we show the induction base holds.
    
    Note that Theorem \ref{p1} directly implies \eqref{2-1} and \eqref{2-4} for $t=T_1$. For \eqref{2-3}, by \eqref{|wi|} we have $h_i(T_1)\leq \|\w_i(T_1)\|\leq \frac{2\|\vv\|}{n}$, and by \eqref{1} we have $h_i(T_1)=\|\w_i(T_1)\|\cos\theta_i(T_1)\geq \|\w_i(T_1)\|/2\geq s_1/2$. For \eqref{2-2}, note that $0\leq h_i(T_1)\leq \|\w_i\|\overset{\eqref{|wi|}}{\leq}\|\vv\|/(3n)~\To~\|\vv\|\geq H(T_1)\geq 2\|\vv\|/3$.
    
    Now suppose \eqref{2-1}, \eqref{2-2} \eqref{2-3} and \eqref{2-4} holds for $T_1, T_1+1,\ldots, t$, next we show the case of $t+1$.
    
\noindent    \textbf{Proof of \eqref{2-1}.}
    First note that due to $\theta_i(t)\leq \epsilon_2$ and $ \frac{2\|\vv\|}{n}\geq h_i(t)\geq \frac{s_1}{2}$ we have 
    \begin{equation}\label{wnorm2}
        \frac{3\|\vv\|}{n}\geq\frac{h_i(t)}{\cos\epsilon_2}\geq\frac{h_i(t)}{\cos\theta_i(t)}=\|\w_i(t)\|\geq h_i\geq\frac{s_1}{2}, \forall i.
    \end{equation}
    
    As computed in \eqref{hdynamics}, $h_i(t+1)=h_i(t)+\frac{\eta}{2}H(t)-Q_i(t), \forall i$.
    
    Note that $\theta_{ij}(t)\leq \theta_i(t)+\theta_j(t)\overset{\eqref{2-4}}{\leq} 2\epsilon_2, \forall i,j$. Similar to \eqref{Qnorm}, we have \begin{equation}\label{Qi}
        |Q_i(t)|\leq \frac{\eta}{2\pi}14\epsilon_2\|\vv\|\leq 3\epsilon_2\eta\|\vv\|\overset{\eqref{2-2}}{\leq}\frac{\eta}{2}\cdot \frac{1}{3}H(t),
    \end{equation} which implies 
    \begin{equation}\label{hb}
        0<\frac{\eta}{2}\cdot \frac{2}{3}H(t)\leq \frac{\eta}{2}H(t)-Q_i(t)\leq \frac{\eta}{2}\cdot \frac{4}{3}H(t), \forall i.
    \end{equation} 
Then $\frac{\eta}{2}H(t)-Q_i(t)\leq 2\left(\frac{\eta}{2}H(t)-Q_j(t)\right), \forall i,j.$
    
    Finally, $\forall i,j$ we have $h_i(t+1)=h_i(t)+\frac{\eta}{2}H(t)-Q_i(t)\leq 2h_j(t)+2\left(\frac{\eta}{2}H(t)-Q_j(t)\right)=2h_j(t+1)$.
   
\noindent    \textbf{Proof of \eqref{2-2}.}  The dynamics of $H(t)$ is given by $H(t+1)=H(t)-\sum_{i}(h_i(t+1)-h_i(t))=H(t)-\sum_{i} \frac{\eta}{2}H(t)+\sum_i Q_i(t)=(1-n\eta/2)H(t)+\sum_i Q_i(t)$.
   
   Note that \eqref{Qi} implies $|\sum_i Q_i(t)|\leq 3n\epsilon_2\|\vv\|\eta$, therefore
   \[H(t+1)-6\epsilon_2\|\vv\|
   \leq\left(1-\frac{n\eta}{2}\right)H(t)+3n\epsilon_2\|\vv\|\eta-6\epsilon_2\|\vv\|=\left(1-\frac{n\eta}{2}\right)\left(H(t)-6\epsilon_2\|\vv\|\right).\]
   
   Iterative application of the above bound yields $H(t+1)-6\epsilon_2\|\vv\|\leq \left(1-\frac{n\eta}{2}\right)^{t+1-T_1}\left(H(T_1)-6\epsilon_2\|\vv\|\right)$.
   
   For the same reason, we also have $H(t+1)+6\epsilon_2\|\vv\|\geq \left(1-\frac{n\eta}{2}\right)^{t+1-T_1}\left(H(T_1)+6\epsilon_2\|\vv\|\right)$.
   
   On the other hand, $\forall i, \|\vv\|/(3n)\overset{\eqref{|wi|}}{\geq}\|\w_i(T_1)\|\geq h_i(T_1)=\|\w_i(T_1)\|\cos\theta_i(T_1)\geq 0$ implies $\|\vv\|\geq H(T_1)\geq 2\|\vv\|/3$. Combining three aforementioned bounds yields the first and second inequality in \eqref{2-2}. 
   
   Now we prove the rightmost inequality in \eqref{2-2}. Note that $n\eta = o(1)\To 1-n\eta/2\geq \exp(-2n\eta/3)$. Then \[\begin{split}
       &\frac{2}{3}\left(1-\frac{n\eta}{2}\right)^{t-T_1}\|\vv\|-6\epsilon_2\|\vv\|\\
       \geq\:& \frac{2}{3}\exp(-2n\eta(T_2-T_1)/3)\|\vv\|-6\epsilon_2\|\vv\| \\
       \geq\:& \frac{2}{3}\exp\left(-\frac{2n\eta}{3}\cdot \frac{3}{2}\frac{1}{n\eta}\ln\left(\frac{1}{36\epsilon_2}\right)\right)\|\vv\|-6\epsilon_2\|\vv\|\\ 
       =\:& \frac{2}{3}\cdot36\epsilon_2\|\vv\|-6\epsilon_2\|\vv\|\\
       =\:&18\epsilon_2\|\vv\|,
   \end{split}\]
where the second inequality is because $T_2 - T_1 = \left\lceil\frac{1}{n\eta}\ln\left(\frac{1}{36\epsilon_2}\right)\right\rceil \leq \frac{3}{2}\frac{1}{n\eta}\ln\left(\frac{1}{36\epsilon_2}\right)$.
   
\noindent    \textbf{Proof of \eqref{2-3}.}  
    Since we have already shown \eqref{2-1} and \eqref{2-2} for $t+1$, $H(t+1)\geq 18\epsilon_2\|\vv\|>0$ implies 
    \[\frac{n}{2}h_i(t+1)\overset{\eqref{2-1}}{\leq} \sum_j h_j(t+1)\leq \|\vv\|, \forall i ~\To~ h_i(t+1)\leq \frac{2}{n}\|\vv\|, \forall i.\]
    
    For the lower bound, by \eqref{hdynamics} and \eqref{hb} we have \begin{equation}\label{hi}
        h_i(t+1)=h_i(t)+\frac{\eta}{2}H(t)-Q_i(t)\geq h_i(t)\geq\frac{s_1}{2}.
    \end{equation}

\noindent    \textbf{Proof of \eqref{2-4}.} Recall that the dynamics of $\cos(\theta_i)$ is given by \eqref{cos-theta} as $\cos(\theta_i(t+1))-\cos(\theta_i(t))=I_1+I_2$.
Then we have
    \[\begin{split}
        I_1=&\frac{\eta}{2\pi\|\w_i(t+1)\|}\left[(\pi-\theta_i(t))\sin^2\theta_i(t)\|\vv\|-\sum_{j\neq i}(\pi-\theta_{ij}(t))\left(\cos\theta_j(t)-\cos\theta_i(t)\cos\theta_{ij}(t)\right)\|\w_j(t)\|\right]\\
        \geq& -\frac{\eta}{2}\sum_{j\neq i} \left(\cos\theta_j(t)-\cos\theta_i(t)\cos\theta_{ij}(t)\right)\frac{\|\w_j(t)\|}{\|\w_i(t+1)\|}\\
        \geq&-\frac{\eta}{2}\sum_{j\neq i} \sin\theta_i(t)\sin(\theta_i(t)+\theta_j(t))\frac{\|\w_j(t)\|}{\|\w_i(t+1)\|},
    \end{split}\]
where the last inequality is because $\theta_{ij}(t)\leq \theta_i(t)+\theta_j(t)\leq 2\epsilon_2<\pi \To 
\cos\theta_j(t)-\cos\theta_i(t)\cos\theta_{ij}(t)\leq \cos\theta_j(t)-\cos\theta_i(t)\cos(\theta_i(t)+\theta_j(t))= \sin\theta_i(t)\sin(\theta_i(t)+\theta_j(t))$.

Since we have already shown \eqref{2-3} for $t+1$, $\|\w_i(t+1)\|\geq h_i(t+1)\overset{\eqref{hi}}{\geq} h_i(t)$ holds. Also we have
$\|\w_j(t)\|=h_j(t)/\cos\theta_j(t)\leq 2h_j(t)$. Then
\[I_1
\geq -\frac{\eta}{2}\sum_{j\neq i} \sin\theta_i(t)(\sin\theta_i(t)+\sin\theta_j(t))\frac{2h_j(t)}{h_i(t)}
\overset{\eqref{2-1}}{\geq} -{2\eta}\sum_{j\neq i} \sin\theta_i(t)(\sin\theta_i(t)+\sin\theta_j(t)).\]

To bound $I_2$, first note that $\|\w_i(t)\|\leq 3\|\vv\|/n, \forall i$. Then by applying elementary triangle inequality in a similar manner as \eqref{gdnorm}, we have $\|\nabla_i(t)\|\leq 5\|\vv\|, \forall i$. Since $\|\w_i(t+1)\|\geq h_i(t+1)\geq s_1/2$, for similar reasons as \eqref{T2bound}, $I_2$ could be lower bounded as
\[I_2\geq -\frac{\eta}{s_1/2}\cdot\frac{\|\nabla_i(t)\|\cdot\eta\|\nabla_i(t)\|+\eta\|\nabla_i(t)\|^2}{s_1}\geq -100\frac{\eta^2\|\vv\|^2}{s_1^2}.\]
So we have \begin{equation}\label{kumar}
    \cos(\theta_i(t+1))-\cos\theta_i(t)\geq- {2\eta}\sum_{j\neq i} \sin\theta_i(t)(\sin\theta_i(t)+\sin\theta_j(t))-100\frac{\eta^2\|\vv\|^2}{s_1^2}.
\end{equation}

Define a potential function $V(t)\coloneqq\sum_i \sin^2\left(\theta_i(t)/2\right)$, we consider the dynamics of $V(t)$:
\begin{equation}
    \begin{split}
        V(t+1)-V(t)
        =&\frac{1}{2}\sum_i\left(\cos\theta_i(t)-\cos\theta_i(t+1)\right)\\
        \leq& \frac{1}{2}\sum_i \left({2\eta}\sum_{j\neq i} \sin\theta_i(t)(\sin\theta_i(t)+\sin\theta_j(t))+100\frac{\eta^2\|\vv\|^2}{s_1^2}\right)\\
        \leq& {\eta}\sum_i \left(\frac{3}{2}n\sin^2\theta_i(t)+\sum_{j} \sin^2\theta_j(t)\right)+50\frac{n\eta^2\|\vv\|^2}{s_1^2}\\
        \leq& {10n\eta}\sum_i \sin^2\left(\frac{\theta_i(t)}{2}\right)+50\frac{n\eta^2\|\vv\|^2}{s_1^2},
    \end{split}
\end{equation}
where the last inequality is because $\sin^2\theta_i(t)=4\sin^2(\theta_i(t)/2)\cos^2(\theta_i(t)/2)\leq 4\sin^2(\theta_i(t)/2)$.

Then we have 
\[V(t+1)+\frac{5\eta\|\vv\|^2}{s_1^2}\leq \left(1+10n\eta\right)\left(V(t)+\frac{5\eta\|\vv\|^2}{s_1^2}\right)\leq\cdots\leq \left(1+10n\eta\right)^{t+1-T_1}\left(V(T_1)+\frac{5\eta\|\vv\|^2}{s_1^2}\right).\]

Note that by \eqref{V(t1)} we have $V(T_1)\leq 2n\epsilon_1^2$, and by setting $\eta =O\left(\frac{\epsilon_1^2\sigma^2d}{\|\vv\|^2}\right)=O\left(\frac{\epsilon_1^2s_1^2}{\|\vv\|^2}\right)$ we have $\frac{5\eta\|\vv\|^2}{s_1^2}\leq n\epsilon_1^2$. Then $V(t+1)\leq \left(1+10n\eta\right)^{t+1-T_1} 3n\epsilon_1^2
\leq \exp(10n\eta(T_2-T_1))3n\epsilon_1^2
\leq\epsilon_2^2/16$.
\end{proof}

\section{Global Convergence: phase 3}\label{Phase 3 of Global Convergence}

\subsection{Initial Condition of Phase 3}
First we prove some initial conditions that are satisfied at time $T_2$, $i.e.$, the start of Phase 3.

\begin{lemma}\label{p3init}
Suppose the conditions \eqref{2-1} \eqref{2-2} \eqref{2-3} \eqref{2-4} in Theorem \ref{p2} holds, then at the start of Phase 3 we have
\[\forall i \in [n], \|\vv\|/(3n)\leq \|\w_i(T_2)\|\leq 3\|\vv\|/n,\]
and 
\[L(\w(T_2))\leq 20\epsilon_2\|\vv\|^2.\]
\end{lemma}

\begin{proof}

\noindent\textbf{Proof of the First Condition.}
By Theorem \ref{p2} we have $\theta_i(T_2)\leq \epsilon_2,\forall i$, and $H(T_2)\leq \left(1-\frac{n\eta}{2}\right)^{T_2-T_1}\|\vv\|+6\epsilon_2\|\vv\|\leq \exp(-\frac{n\eta}{2}(T_2-T_1))\|\vv\|+6\epsilon_2\|\vv\|=(36\epsilon_2)^{1/2}\|\vv\|+6\epsilon_2\|\vv\|\leq 7\epsilon_2^{1/2}\|\vv\|.$

Then for $\forall i\in [n]$, $\frac{2}{3}\|\vv\|\leq \|\vv\|-H(T_2)=\sum_j h_j(T_2)\leq 2nh_i(T_2) \To \|\w_i(T_2)\|\geq h_i(T_2)\geq \|\vv\|/(3n)$. 

Similarly, for $\forall i\in [n]$, $H(T_2)\geq 0\To \|\vv\|\geq \sum_j h_j(T_2)\geq nh_i(T_2)/2 \To$
    $ h_i(T_2)\leq 2\|\vv\|/n \To \|\w_i(T_2)\|=h_i(T_2)/\cos(\theta_i(T_2))\leq \frac{3}{2}h_i(T_2)\leq 3\|\vv\|/n$.

\noindent\textbf{Proof of the Second Condition.}
    Since $\|\vv\|/(3n)\leq h_i(T_2)\leq 2\|\vv\|/n, \forall i,$ we have
    \[ \begin{split}
    \left\|\sum_i \w_i(T_2) - \vv\right\|&\leq \sum_i\|\w_i(T_2)-h_i(T_2)\overline{\vv}\|+\left\|\sum_i h_i(T_2)\overline{\vv}-\vv\right\|\\
    &=\sum_ih_i(T_2)\tan\theta_i(T_2)+H(T_2)\leq 8\epsilon_2^{1/2}\|\vv\|,
    \end{split}\]
where the last inequality is because $\forall i, \theta_i(T_2)\leq \epsilon_2 = o(1)\To \tan\theta_i(T_2) \leq 2\epsilon_2\leq o(\epsilon_2^{1/2}).$

    So according to \eqref{loss} we have
    \[L(\w(T_2))\leq \frac{1}{4}\left(8\epsilon_2^{1/2}\|\vv\|\right)^2+\frac{1}{2\pi}\left(n^2\cdot 2\epsilon_2\cdot \left(\frac{2}{n}\|\vv\|\right)^2+n\epsilon_2\frac{2}{n}\|\vv\|^2\right)\leq 20\epsilon_2\|\vv\|^2.\]
    
\end{proof}

\subsection{Proofs for Gradient Lower Bound}\label{Proofs for Gradient Lower Bound}

Before proving Theorem \ref{gdbound}, we need some auxiliary lemmas.

\subsubsection{Auxiliary lemmas}
\begin{replemma}{gd_projection}
    Recall the global minimum $\w_1^*, \w_2^*, \ldots, \w_n^*$ defined as $\w_i^*=\frac{h_i}{\sum_{j\in [n]} h_j}\vv$. Define $ \theta_{\max}:=\max_{i\in[n]} \theta_i$, then
    \[\sum_{i=1}^n\left\langle\frac{\partial}{\partial \w_i}L(\w), \w_i-\w_i^*\right\rangle\geq 2L(\w)-O\left(\theta_{\max}^2\|\rr\|\cdot\|\vv\|\right).\]
\end{replemma}
\begin{proof}
First we introduce the idea of residual decomposition in \cite{zhou2021local}, which decomposes the residual function $R(\x)$ in two terms :  
\[R(\x)=\sum_{j=1}^n\left[\bm{w}_j^\top \bm{x}\right]_+-\left[\bm{v}^\top \bm{x}\right]_+
=\rr^\top\x\cdot\mathbb{1}\{\vv^\top\x\geq 0\}
+\sum_{j=1}^n\bm{w}_j^\top \bm{x}\left(\mathbb{1}\{\w_j^\top\x\geq 0\}-\mathbb{1}\{\vv^\top\x\geq 0\}\right).\]

Define $R_1(\x)=\rr^\top\x\cdot\mathbb{1}\{\vv^\top\x\geq 0\}$ and $R_2(\x)=\sum_{j=1}^n\bm{w}_j^\top \bm{x}\left(\mathbb{1}\{\w_j^\top\x\geq 0\}-\mathbb{1}\{\vv^\top\x\geq 0\}\right)$, then $R(\x)=R_1(\x)+R_2(\x)$.

Back to the lemma, first we have the following algebraic calculations:
\[\begin{split}
        &\sum_{i=1}^n \left\langle\frac{\partial}{\partial \w_i}L(\w), \w_i-\w_i^*\right\rangle\\
        &=\sum_{i=1}^n\mathbb{E}_{\bm{x}}\left[R(\x)\mathbb{1}\left\{\w_i^\top \x\geq 0\right\}\x^\top(\w_i-\w_i^*)\right]\\
        &=\mathbb{E}_{\bm{x}}\left[R(\x)\sum_{i=1}^n\left(\left[\bm{w}_i^\top \bm{x}\right]_+-\mathbb{1}\left\{\w_i^\top \x\geq 0\right\}\x^\top\w_i^*\right)\right]\\
        &=2L(\w)+\mathbb{E}_{\bm{x}}\left[R(\x)\sum_{i=1}^n\left(\mathbb{1}\left\{{\w_i^*}^\top \x\geq 0\right\}-\mathbb{1}\left\{\w_i^\top \x\geq 0\right\}\right)\x^\top\w_i^*\right]
    \end{split}\]
    
With the residual decomposition, the last term above can be decomposed into two terms $I_1$, $I_2$ as 
    \[\begin{split}
        &\mathbb{E}_{\bm{x}}\left[R(\x)\sum_{i=1}^n\left(\mathbb{1}\left\{{\w_i^*}^\top \x\geq 0\right\}-\mathbb{1}\left\{\w_i^\top \x\geq 0\right\}\right)\x^\top\w_i^*\right]\\
        =&\underbrace{\mathbb{E}_{\bm{x}}\left[R_1(\x)\sum_{i=1}^n\left(\mathbb{1}\left\{{\w_i^*}^\top \x\geq 0\right\}-\mathbb{1}\left\{\w_i^\top \x\geq 0\right\}\right)\x^\top\w_i^*\right]}_{I_1}\\
        &+\underbrace{\mathbb{E}_{\bm{x}}\left[R_2(\x)\sum_{i=1}^n\left(\mathbb{1}\left\{{\w_i^*}^\top \x\geq 0\right\}-\mathbb{1}\left\{\w_i^\top \x\geq 0\right\}\right)\x^\top\w_i^*\right]}_{I_2}.
    \end{split}\]

    For the second term $I_2$, note that
    \[\forall j, \bm{w}_j^\top \bm{x}\left(\mathbb{1}\{\w_j^\top\x\geq 0\}-\mathbb{1}\{\vv^\top\x\geq 0\}\right)\geq 0 \To R_2(\x)\geq 0\]
    and $\forall \w_i^*, \w_i,$
    \[\left(\mathbb{1}\left\{{\w_i^*}^\top \x\geq 0\right\}-\mathbb{1}\left\{\w_i^\top \x\geq 0\right\}\right)\x^\top\w_i^*\geq 0,\]
    so $I_2$ is always non-negative.
    
    For the first term $I_1=\sum_{ i\in [n]}\mathbb{E}_{\bm{x}}\left[\rr^\top\x\mathbb{1}(\vv^\top\x\geq 0)\left(\mathbb{1}\left\{{\w_i^*}^\top \x\geq 0\right\}-\mathbb{1}\left\{\w_i^\top \x\geq 0\right\}\right)\x^\top\w_i^*\right]$, we bound each term in the summation as
    \[\begin{split}
        &\mathbb{E}_{\bm{x}}\left[\rr^\top\x\mathbb{1}(\vv^\top\x\geq 0)\left(\mathbb{1}\left\{{\w_i^*}^\top \x\geq 0\right\}-\mathbb{1}\left\{\w_i^\top \x\geq 0\right\}\right)\x^\top\w_i^*\right]\\
        &\geq - \mathbb{E}_{\bm{x}}\left[|\rr^\top\x|\cdot\left|\mathbb{1}\left\{{\w_i^*}^\top \x\geq 0\right\}-\mathbb{1}\left\{\w_i^\top \x\geq 0\right\}\right|\cdot\|\tilde{\x}\|\cdot\|\w_i^*\|\theta_{i}\right]\\
        &= -\|\rr\|\theta_{i}\|\w_i^*\|\cdot \mathbb{E}_{\tilde{\x}}\left[\|\tilde{\x}\|^2\left|\mathbb{1}\left\{{\w_i^*}^\top \tilde{\x}\geq 0\right\}-\mathbb{1}\left\{\w_i^\top \tilde{\x}\geq 0\right\}\right|\right]\\
        &\geq-O\left(\|\rr\|\theta_{i}^2\|\w_i^*\|\right)
    \end{split}\]
where $\tilde{\x}$ is the projection of $\x$ onto $span(\w_i^*, \w_i, \rr)$ and follows a three-dimensional Gaussian. Here the first inequality is because $\mathbb{1}\left\{{\w_i^*}^\top \x\geq 0\right\}-\mathbb{1}\left\{\w_i^\top \x\geq 0\right\}\neq 0 \To \theta(\w_i^*,\x)\in [\pi/2-\theta_i, \pi/2+\theta_i] \To |\x^\top \w_i^*| =|\tilde{\x}^\top \w_i^*| \leq \|\tilde{\x}\|\cdot\|\w_i^*\|\theta_{i}$, and the last inequality is because \[\mathbb{E}_{\tilde{\x}}\left[\|\tilde{\x}\|^2\left|\mathbb{1}\left\{{\w_i^*}^\top \tilde{\x}\geq 0\right\}-\mathbb{1}\left\{\w_i^\top \tilde{\x}\geq 0\right\}\right|\right]=O(\theta_i).\] (See Lemma C.5 in \cite{zhou2021local} for detailed calculations.)
    
Note that $\sum_i \|\w_i^*\| = \|\vv\|$, so we have  
\[\sum_{i=1}^n\left\langle\frac{\partial}{\partial \w_i}L(\w), \w_i-\w_i^*\right\rangle= 2L(\w)+I_1+I_2\geq 2L(\w)-\sum_{i\in[n]}O\left(\|\rr\|\theta_{i}^2\|\w_i^*\|\right)\geq 2L(\w)-O\left(\theta_{\max}^2\|\rr\|\cdot\|\vv\|\right).\]
\end{proof}

\begin{lemma}[Bound of $\theta_i$]\label{theta} We have that
    \[\|\w_i\|^2\theta_i^3\leq 30\pi L(\w), \forall i.\]
\end{lemma}
\begin{proof}
    W.L.O.G., suppose $\vv=(\|\vv\||,0,\ldots,0)^\top$ and $\w_i =(\cos\theta_i, \sin\theta_i, 0,\ldots,0)^\top\|\w_i\|$.
     
    Define $S_i\coloneqq\{\x:\x\in\mathbb{R}^d, \x^\top \w_i\geq0\wedge\x^\top \vv<0\}$. One can see that $S_i=\{\x:\theta(\x,\w_i)\leq \pi/2,\theta(\x,\vv)\geq \pi/2\}$.
    On the other hand, $\forall \x\in S_i$, $R(\x)=\sum_{j=1}^n\left[\bm{w}_j^\top \bm{x}\right]_+-\left[\bm{v}^\top \bm{x}\right]_+\geq\left[\bm{w}_i^\top \bm{x}\right]_+\geq 0$. Therefore,
    \[\begin{split}
        L(\w)&=\mathbb{E}_{\bm{x}\sim\mathcal{N}(\bm{0}, I)})\left[\frac{1}{2}\left(\sum_{i=1}^n\left[\bm{w}_i^\top \bm{x}\right]_+-\left[\bm{v}^\top \bm{x}\right]_+\right)^2\right]\\
        &\geq\int_{\x\in S_i}\frac{1}{2}( \w_i^\top\x)^2 \frac{e^{-\frac{\|\x\|^2}{2}}}{(2\pi)^{d/2}}d x\\
        &=\int_{\rho=0}^{+\infty}\int_{\omega=\pi/2}^{\theta_i+\pi/2}\frac{1}{4\pi}(\|\w_i\|\rho\cos(\omega-\theta_i))^2\rho e^{-\rho^2/2}d\omega d\rho\\
        &=\frac{\|\w_i\|^2}{2\pi}\int_{\omega=\pi/2}^{\theta_i+\pi/2}\cos^2(\omega-\theta_i)d\omega\\
        &=\frac{\|\w_i\|^2}{8\pi}\left(2\theta_i-\sin(2\theta_i)\right)\\
        &\geq \frac{\|\w_i\|^2}{8\pi}\cdot \frac{(2\theta_i)^3}{30}.
    \end{split}\]

Rearranging terms yields the result.
\end{proof}

\begin{lemma}[Bound of $\|\rr\|$]\label{r}
Given that $\frac{4\|\vv\|}{n}\geq \|\w_i\|\geq \frac{\|\vv\|}{4n}$ for all $i\in[n]$ and $L(\w)=O(n^{-2}\|\vv\|^2)$, then
\[\|\rr\|=O\left(nL^{1/2}(\w)\right).\]
\end{lemma}

\begin{proof}
Lemma \ref{theta} and $\|\w_i\|=\Theta(\|\vv\|/n)$ implies $\theta_i=O\left(n^{2/3}\left(\frac{L(\w)}{\|\vv\|^2}\right)^{1/3}\right)=o(1).$ For all $i\in [n]$, 
 by Taylor expansion we have $|(\sin\theta_i-\theta_i\cos\theta_i)| = O(\theta_i^3)$, then $|(\sin\theta_i-\theta_i\cos\theta_i)\|\w_i\|\cdot\|\bm{v}\||=O(\theta_i^3)\|\w_i\|\cdot\|\bm{v}\|=O\left(n{L(\w)}\right)$. Similarly, $\forall i,j\in[n]$ we have $\theta_{ij}\leq \theta_i+\theta_j\leq O\left(n^{2/3}\left(\frac{L(\w)}{\|\vv\|^2}\right)^{1/3}\right)\To |(\sin\theta_{ij}-\theta_{ij}\cos\theta_{ij})\|\w_i\|\cdot\|\w_j\||=O(\theta_{ij}^3)\|\w_i\|\cdot\|\w_{j}\|=O\left({L(\w)}\right)$.

Then by \eqref{loss}, we have
\[\begin{split}
    \|\rr\|^2&=4L(\w)-\frac{2}{\pi}\left[\sum_{i< j}(\sin\theta_{ij}-\theta_{ij}\cos\theta_{ij})\|\w_i\|\cdot\|\w_j\|-\sum_i(\sin\theta_i-\theta_i\cos\theta_i)\|\w_i\|\cdot\|\bm{v}\|\right]\\
    &\leq4L(\w)+ n^2O(L(\w))+nO(nL(\w))\leq O(n^2L(\w)),
\end{split}\]
which implies $\|\rr\|=O\left(nL^{1/2}(\w)\right)$.
\end{proof}

\begin{lemma}[Bound of $\|\w_i-\w_i^*\|$]\label{w-w*}
    Suppose $\frac{4\|\vv\|}{n}\geq \|\w_i\|\geq \frac{\|\vv\|}{4n}, \forall i\in[n]$ and $L(\w)=O(\|\vv\|^2/n^2)$. 
    Then $\|\w_i -\w_i^*\|\leq O\left(n^{2/3}\left(\frac{L(\w)}{\|\vv\|^2}\right)^{1/3}\right)\|\w_i\| $.
\end{lemma}
\begin{proof}
Lemma \ref{theta} and $\|\w_i\|=\Theta(\|\vv\|/n)$ implies $\theta_i=O\left(n^{2/3}\left(\frac{L(\w)}{\|\vv\|^2}\right)^{1/3}\right).$
Lemma \ref{r} implies $|H|=|\langle \rr, \overline{\vv}\rangle|=O(nL^{1/2}(\w))$.
    
    We first decompose $\|\w_i -\w_i^*\|$ into two parts as $\|\w_i -\w_i^*\|\leq \|\w_i-h_i\overline{\vv}\|+\|h_i\overline{\vv}-\w_i^*\|$.
    
    The first part can be bounded as $\|\w_i-h_i\overline{\vv}\|=\|\w_i\|\sin\theta_i\leq O\left(n^{2/3}\left(\frac{L(\w)}{\|\vv\|^2}\right)^{1/3}\right)\|\w_i\|$.
    
    The second part can be bounded as $\|h_i\overline{\vv}-\w_i^*\|=\left\vert h_i\left(1-\frac{\|\vv\|}{\sum_j h_j}\right)\right\vert=h_i\frac{|H|}{\|\vv\|-H}\leq \|\w_i\|\frac{|H|}{\|\vv\|-H}.$ Note that $|H|=|\langle \rr, \overline{\vv}\rangle|\leq \|\rr\|=O(nL^{1/2}(\w))\leq O(\|\vv\|)\To \|\vv\|-|H|\geq \|\vv\|/2$. So we have $\|h_i\overline{\vv}-\w_i^*\|\leq \|\w_i\|\frac{|H|}{\|\vv\|-H}\leq \|\w_i\|\cdot\frac{1}{\|\vv\|/2}O(nL^{1/2}(\w))\leq O\left(n^{2/3}\left(\frac{L(\w)}{\|\vv\|^2}\right)^{1/3}\right)\|\w_i\|.$
    
    Combining two parts together yields the bound.
\end{proof}

\subsubsection{Proof of Theorem \ref{gdbound}}
Now we are ready to prove Theorem \ref{gdbound}.
\begin{reptheorem}{gdbound}
    If for every student neuron we have $\frac{4\|\vv\|}{n}\geq \|\w_i\|\geq \frac{\|\vv\|}{4n}$, and 
    \[L(\w)= O\left(\frac{\|\vv\|^2}{n^{14}}\right),\] then $\left\|\nabla_{\w}L(\w)\right\|\geq \Omega\left(\frac{{L^{2/3}(\w)}}{n^{2/3}\|\vv\|^{1/3}}\right)$.
\end{reptheorem}

\begin{proof}
   Lemma \ref{theta} and $\|\w_i\|=\Theta(\|\vv\|/n)$ implies $\theta_i=O\left(n^{2/3}\left(\frac{L(\w)}{\|\vv\|^2}\right)^{1/3}\right).$
Lemma \ref{r} implies $\|\rr\|=O(nL^{1/2}(\w))$.

    Combined with lemma \ref{gd_projection} and $L(\w)= O\left(\frac{\|\vv\|^2}{n^{14}}\right)$ we have
    \[\sum_{i=1}^n\left\langle\frac{\partial}{\partial \w_i}L(\w), \w_i-\w_i^*\right\rangle\geq2L(\w)-O\left(\theta_{\max}^2\|\rr\|\cdot\|\vv\|\right)
    \geq 2L(\w)-O(n^{7/3}L^{7/6}(\w)\|\vv\|^{-1/3})
    \geq L(\w).\]
    
    Then
    \[\begin{split}
        L(\w)&\leq \sum_{i=1}^n\left\langle\frac{\partial}{\partial \w_i}L(\w), \w_i-\w_i^*\right\rangle
        \leq \sum_{i=1}^n\left\|\frac{\partial}{\partial \w_i}L(\w)\right\|\cdot \left\|\w_i-\w_i^*\right\|
        \leq \left\|\frac{\partial L(\w)}{\partial \w}\right\|\sum_{i\in[n]} \|\w_i -\w_i^*\|\\
        &\overset{\text{Lemma \ref{w-w*}}}{\leq}  \left\|\frac{\partial L(\w)}{\partial \w}\right\| O\left(n^{2/3}\left(\frac{L(\w)}{\|\vv\|^2}\right)^{1/3}\right)\sum_{i\in[n]}\|\w_i^*\|=O\left(n^{2/3}{L(\w)}^{1/3}{\|\vv\|^{1/3}}\right)\left\|\frac{\partial L(\w)}{\partial \w}\right\|.
    \end{split}\]
    So $\left\|\frac{\partial L(\w)}{\partial \w}\right\|\geq \Omega\left(\frac{{L^{2/3}(\w)}}{n^{2/3}\|\vv\|^{1/3}}\right)$.    
\end{proof}

\subsection{Handling Non-smoothness}
In this section, we establish two lemmas needed for handling the non-smoothness of $L$.

Define the Hessian matrix of $L$ as $\Lambda:=\frac{\partial^2 L(\w)}{\partial \w^2}$. The next lemma ensures the smoothness of $L$ when the student neurons are regularized, $i.e$, their norms are upper and lower bounded.

\begin{lemma}[Conditional Smoothness of $L$]\label{smoothness}
    If for every student neuron we have $\frac{4\|\vv\|}{n}\geq \|\w_i\|\geq \frac{\|\vv\|}{4n}$, then $\|\Lambda\|_2\leq O(n^2)$.
\end{lemma}
\begin{proof}
   When $\w_i\neq \bm{0}, \forall i$, \citet{safran20effects} has shown that $L(\w)$ is twice differentiable and computed the closed form expression of Hessian $\Lambda=\frac{\partial^2 L}{\partial \w^2}\in\mathbf{R}^{nd\times nd}$:
   \begin{equation}\label{Hessian}
       \Lambda=\begin{pmatrix}
       \Lambda_{1,1}&\cdots&\Lambda_{1,n}\\
       \vdots&&\vdots\\
       \Lambda_{n,1}&\cdots&\Lambda_{n,n}
   \end{pmatrix},
   \end{equation}
where $\Lambda_{i,j}\in \mathbb{R}^{d\times d}, i,j\in[n]$ are $d\times d$ matrices with the following forms:

The $i^{\text{th}}$ diagonal block matrix of $\Lambda$ is
\[\Lambda_{i,i}=\frac{1}{2}I+\sum_{j\neq i} \zeta(\w_i, \w_j)-\zeta(\w_i, \vv),\]
where 
\[\zeta(\w, \vv)=\frac{\sin\theta(\w,\vv)\|\vv\|}{2\pi\|\w\|}(I-\overline{\w}\overline{\w}^\top+\overline{\bm{n}}_{\bm{\vv,\w}}\overline{\bm{n}}_{\bm{\vv,\w}}^\top),\]
and ${\bm{n}}_{\bm{\vv,\w}}=\overline{\vv}-\cos\theta(\w,\vv)\overline{\w}$.

For $i\neq j$, the off-diagonal entry is
\[\Lambda_{i,j}=\frac{1}{2\pi}\left[(\pi-\theta(\w_i, \w_j))I+\overline{\bm{n}}_{\bm{\w_i,\w_j}}\overline{\w}_j^\top+\overline{\bm{n}}_{\bm{\w_j,\w_i}}\overline{\w}_i^\top\right].\]

Note that $\|\w_i\|=\Theta(\frac{\|\vv\|}{n}), \forall i$ implies $\frac{\|\w_j\|}{\|\w_i\|}=O(1), \forall i, j$ and $\frac{\|\vv\|}{\|\w_i\|}=O(n), \forall i$. Then $\|\zeta(\w_i,\w_j)\|\leq\frac{\|w_j\|}{2\pi\|\w_i\|}=O(1)$ and $\|\zeta(\w_i,\vv)\|\leq O(1)\frac{\|\vv\|}{\|\w_i\|}\leq O(n)$, so $\|\Lambda_{i,i}\|\leq\|\frac{1}{2}I\|+\sum_{j\neq i} \|\zeta(\w_i, \w_j)\|+\|\zeta(\w_i, \vv)\|\leq O(n)$.

Also note that $\|\Lambda_{i,j}\|\leq \frac{1}{2\pi}(\pi+1+1)\leq 1$ for all $i\neq j$.

Then $\|\Lambda(\w)\|\leq \sum_{i,j}\|\Lambda_{i,j}\|\leq nO(n)+(n^2-n)\leq O(n^2)$.
\end{proof}
The following lemma shows that each student neuron $\w_i$ will not move too far in the third phase.
\begin{lemma}[Bound of the Change of Neurons]\label{locality}
    If the initial loss at phase 3 is upper bounded by $L(\w(T_2))\leq C_l$, and there exists constant $C_s>0$ such that $L(\w(t+1))\leq L(\w(t))-\frac{\eta}{2}\|\nabla_W(t)\|^2\leq L(\w(t))-
    C_s\eta L^{4/3}(\w(t)), \forall T+T_2-1\geq t\geq T_2$, then
    \[L(T+T_2)\leq \frac{1}{(L(\w(T_2))^{-1/3}+C_s\eta T/3)^3},\]
    and
    \[ \sum_{t=0}^{T-1} \eta\|\nabla_W(t+T_2)\|\leq 8C_s^{-1/2}C_l^{1/3}.\]
\end{lemma}
\begin{proof}
    We bound the loss as
    \[\begin{split}
        &~~~~~~\frac{1}{L^{1/3}(\w (t+1))}\\
        &\geq\frac{1}{\left(L(\w(t))-C_s\eta L^{4/3}(\w(t))\right)^{1/3}}\\
        &=\frac{1}{L^{1/3}(\w(t))}
        \left(1+\frac{1-\left(1-C_s\eta L^{1/3}(\w(t))\right)^{1/3}}{\left(1-C_s\eta L^{1/3}(\w(t))\right)^{1/3}}\right)\\
        &=\frac{1}{L^{1/3}(\w(t))}
        \left(1+\frac{C_s\eta L^{1/3}(\w(t))}{\left(1-C_s\eta L^{1/3}(\w(t))\right)^{1/3}\left(1+\left(1-C_s\eta L^{1/3}(\w(t))\right)^{1/3}+\left(1-C_s\eta L^{1/3}(\w(t))\right)^{2/3}\right)}\right)\\
        &=\frac{1}{L^{1/3}(\w(t))}+C_s\eta\frac{1}{\left(1-C_s\eta L^{1/3}(\w(t))\right)^{1/3}+\left(1-C_s\eta L^{1/3}(\w(t))\right)^{2/3}+\left(1-C_s\eta L^{1/3}(\w(t))\right)}\\
        &\geq\frac{1}{L^{1/3}(\w(t))}+\frac{C_s\eta}{3}.
    \end{split}\]
    
    Therefore, ${L^{-1/3}(\w (t+T_2))}\geq {L^{-1/3}(\w (T_2))}+\frac{C_s\eta}{3}t, \forall T\geq t\geq 0.$ Let $l_1\coloneqq {L^{-1/3}(\w(T_2))}$, then $1/l_1^3\leq C_l$ and \[L(\w (T_2+t))\leq\frac{1}{(l_1+C_s\eta t/3)^3}, \forall t\leq T,\]
    this proves the first inequality.

    For the second inequality, note that
    \[L(\w(t+1))\leq L(\w(t))-\frac{\eta}{2}\|\nabla_W(t)\|^2\To \|\nabla_W(t)\|^2\leq \frac{2}{\eta}\left(L(\w(t))-L(\w(t+1))\right).\]
    
By Cauchy inequality, $\forall T>0,$
    \[\begin{split}
        &\left(\sum_{t=0}^{T-1} \|\nabla_W(T_2+t)\|\right)^2\\
        \leq&\left(\sum_{t=0}^{T-1} \frac{1}{(l_1+C_s\eta t/3)^2}\right)\left(\sum_{t=0}^{T-1} {(l_1+C_s\eta t/3)^2}\|\nabla_W(T_2+t)\|^2\right)\\
        \leq&\left(\sum_{t=0}^{T-1} \frac{1}{(l_1+C_s\eta t/3)^2}\right)\left(\sum_{t=0}^{T-1} {(l_1+C_s\eta t/3)^2}\frac{2}{\eta}\left(L(\w(T_2+t))-L(\w(T_2+t+1))\right)\right)\\
        \leq&\frac{2}{\eta}\left(\sum_{t=0}^{T-1} \frac{1}{(l_1+C_s\eta t/3)^2}\right)\left(l_1^2L(\w(T_2+T))+\sum_{t=1}^{T-1} \left((l_1+C_s\eta t/3)^2-{(l_1+C_s\eta (t-1)/3)^2}\right)L(\w(T_2+t))\right)\\
        \leq&\frac{2}{\eta}\left(\sum_{t=0}^{T-1} \frac{1}{(l_1+C_s\eta t/3)^2}\right)\left(l_1^2L(\w(T_2+T))+\frac{2C_s\eta}{3}\sum_{t=1}^{T-1} \left(l_1+C_s\eta t/3\right)L(\w(T_2+t))\right)\\
        \leq&\frac{2}{\eta}\left(\sum_{t=0}^{T-1} \frac{1}{(l_1+C_s\eta t/3)^2}\right)\left(\frac{1}{l_1}+\frac{2C_s\eta}{3}\sum_{t=1}^{T-1} \frac{1}{(l_1+C_s\eta t/3)^2}\right)\\
    \end{split}\]
    
    Note that 
    \[\sum_{t=0}^{T-1} \frac{1}{(l_1+C_s\eta t/3)^2}\leq \sum_{t=0}^{+\infty} \frac{1}{(l_1+C_s\eta t/3)^2}
    \leq\frac{3}{C_s\eta}\sum_{t=0}^{+\infty} \left(\frac{1}{l_1+C_s\eta (t-1)/3}-\frac{1}{l_1+C_s\eta t/3}\right)\leq \frac{6}{C_s\eta l_1}.\]
    
    Therefore,
    \[\left(\sum_{t=0}^{T-1} \|\nabla_W(T_2+t)\|\right)^2
    \leq\frac{2}{\eta}\left(\frac{6}{C_s\eta l_1}\right)\left(\frac{1}{l_1}+\frac{2C_s\eta}{3}\frac{6}{C_s\eta l_1}\right)\leq \frac{2}{\eta}\frac{6}{C_s\eta l_1}\frac{5}{l_1}=\frac{60}{C_s\eta^2l_1^2}.\]
    
    So we get
    \[\eta \sum_{t=0}^T \|\nabla(T_2+t)\|\leq \eta\sqrt{\frac{60}{C_s\eta^2l_1^2}}\leq 8C_s^{-1/2}C_l^{1/3}.\]
\end{proof}

\subsection{Proof of Theorem \ref{p3}}\label{proof of p3}

\begin{reptheorem}{p3}
Suppose the initial condition in Lemma \ref{Initialization} holds. If we set $\epsilon_2=O(n^{-14})$ in Theorem \ref{p2}, $\eta=O\left(\frac{1}{n^2}\right)$, then $\forall T\in \mathbf{N}$ we have
\begin{equation}\label{regularization}
        \frac{4\|\vv\|}{n}\geq \|\w_i(T+T_2)\|\geq \frac{\|\vv\|}{4n}, 
    \end{equation}
and
    \begin{equation}\label{loss_bound}
        L(T+T_2)\leq O\left(\frac{n^4\|\vv\|^2}{\left(\eta T \right)^3}\right).
    \end{equation}

\end{reptheorem}
\begin{proof}
Since we have set $\epsilon_2=O(n^{-14})$, by Lemma \ref{p3init} we have
$\frac{3\|\vv\|}{n}\geq \|\w_i(T_2)\|\geq \frac{\|\vv\|}{3n}, \forall i$, and $L(\w(T_2))\leq 20\epsilon_2\|\vv\|^2 = O\left(\frac{\|\vv\|^2}{n^{14}}\right)$.

To prove the theorem, we just need to prove \eqref{regularization} and a stronger version of \eqref{loss_bound}:
\begin{equation}\label{convergence}
        L(T+T_2)\leq \frac{1}{\left(L(T_2)^{-1/3}+\Omega\left(\frac{{1}}{n^{4/3}\|\vv\|^{2/3}}\right)\eta T \right)^3}.
    \end{equation}

We prove \eqref{regularization} and \eqref{convergence} together inductively.

The induction base holds for $T=0$ by Lemma \ref{p3init}.
    
    Now suppose  \eqref{regularization} \eqref{convergence} hold for $0,1,\ldots,T-1$, we show the case of $T$.
    
   First note that  \eqref{gradient} \eqref{regularization} and routine computation implies
    $\left\|\nabla_i(t+T_2)\right\|=O(\|\vv\|), \forall 0\leq t\leq T-1.$
    
    For $\forall T_2\leq t \leq  T+T_2-1$, since the induction condition together with lemma \ref{smoothness} guarantee the smoothness of $L$, the classical analysis of gradient descent can be applied ([\cite{nesterov2018lectures}], lemma 1.2.3) to bound the decrease of loss at time $t$ as
    \begin{equation}\label{9}
        \begin{split}
        L(\w(t+1))=&L(\w(t))+\left\langle\nabla_W(t), -\eta \nabla_W(t)\right\rangle+\\
        &\int_{\tau=0}^1(1-\tau)(-\eta \nabla_W(t))^\top\frac{\partial^2L}{\partial \w^2}(\w(t)-\tau\eta \nabla_W(t))(-\eta \nabla_W(t))\mathrm{d}\tau.
    \end{split}
    \end{equation}

For $\forall \tau\in[0,1]$, $\|\w_i(t)-\tau\eta \nabla_i(t)\|\geq \|\w_i(t)\|-\eta\|\nabla_i(t)\|\geq \frac{\|\vv\|}{4n}-\eta O(\|\vv\|)\geq \frac{\|\vv\|}{5n}$, similarly we have $\|\w_i(t)-\tau\eta \nabla_i(t)\|\leq \|\w_i(t)\|+\eta\|\nabla_i(t)\|\leq \frac{5\|\vv\|}{n}$. 
Then lemma \ref{smoothness} implies the smoothness of $L$ at $\w(t)-\tau\eta \nabla_W(t)$:
$\left\|\frac{\partial^2L}{\partial \w^2}(\w(t)-\tau\eta \nabla_W(t))\right\|\leq O\left(n^2\right)$.
Combined with gradient lower bound Theorem \ref{gdbound} (note that $L(\w(t))=O(\|\vv\|^2/n^{14})$), the dynamic of loss can be bounded as
\[\begin{split}
    L(\w(t))-L(\w(t+1))&\geq \eta\|\nabla_W(t)\|^2-\int_{\tau=0}^1(1-\tau)O\left(n^2\right)\|-\eta \nabla_W(t)\|^2\mathrm{d}\tau\\
    &\geq \frac{\eta}{2}\|\nabla_W(t)\|^2\\
    &\overset{\text{Theorem }\ref{gdbound}}{\geq}
    \Omega\left(\frac{{1}}{n^{4/3}\|\vv\|^{2/3}}\right)\eta L^{4/3}(\w(t)).
\end{split}\]

Set $C_s$ in Lemma \ref{locality} as $C_s= \Omega\left(\frac{{1}}{n^{4/3}\|\vv\|^{2/3}}\right)$. For $\forall T+T_2-1\geq t\geq T_2$, the above inequality implies that $L(\w(t+1))\leq L(\w(t))-\frac{\eta}{2}\|\nabla_W(t)\|^2\leq L(\w(t))-
    C_s\eta L^{4/3}(\w(t))$. So we can apply Lemma \ref{locality} here, which immediately implies \eqref{convergence}. 

For \eqref{regularization}, Lemma \ref{locality} yields
\[\begin{split}
\|\w_i(T+T_2)\|&\geq \|\w_i(T_2)\|-\sum_{t=0}^{T-1} \eta\|\nabla_W(t+T_2)\|\geq \frac{\|\vv\|}{3n}-8C_s^{-1/2}L(\w(T_2))^{1/3}\\
&=\frac{\|\vv\|}{3n}- O\left(n^{2/3}\|\vv\|^{1/3}\cdot \left(\frac{\|\vv\|^2}{n^{14}}\right)^{1/3}\right)\geq \frac{\|\vv\|}{4n},
\end{split}\]
similarly
\[\|\w_i(T+T_2)\|\leq \|\w_i(T_2)\|+\sum_{t=0}^{T-1} \eta\|\nabla_W(t+T_2)\|
\leq \frac{3\|\vv\|}{n}+O\left(n^{2/3}\|\vv\|^{1/3}\cdot \left(\frac{\|\vv\|^2}{n^{14}}\right)^{1/3}\right)\leq \frac{4\|\vv\|}{n}.\]
\end{proof}

\section{Supplementary Materials for Section \ref{Proof Sketch}}
\subsection{Proof of Lemma \ref{Initialization}}\label{proof of Initialization}
\begin{replemma}{Initialization}
Let $s_1 \coloneqq \frac{1}{2}\sigma \sqrt{d}, s_2\coloneqq 2\sigma \sqrt{d}$.    When $d=\Omega(\log(n/\delta))$, with probability at least $1-\delta$, the following properties holds:
    \begin{equation}\label{w(0)}
        \forall i\in[n], s_1\leq \|\w_i(0)\|\leq s_2,
    \end{equation}
    \begin{equation}\label{theta(0)}
        \forall i\in [n], \frac{\pi}{3}\leq \theta_i(0)\leq \frac{2\pi}{3}.
    \end{equation}
\end{replemma}

\begin{proof}
By concentration inequality of Gaussian (See Section 2 in \cite{Dasgupta13}), For $\forall i$ we have $\Pr\left[\|\w_i(0)\|<\frac{1}{2}\sigma\sqrt{d}\vee \|\w_i(0)\|>2\sigma\sqrt{d}\right]\leq \Pr\left[|\|\w_i(0)\|^2-\sigma^2d|>\frac{3}{4}\sigma^2d\right]\leq \exp\left(-\left(\frac{3}{4}\right)^2d/24\right)\leq \exp(-\Omega(\log(n/\delta)))\leq \frac{\delta}{3n}.$ By union bound, \eqref{w(0)} holds with probability at least $1-\delta/3$.

For \eqref{theta(0)}, note that for $\forall i\in[n]$, 
\[|\langle \w_i(0), \overline{\vv}\rangle| \leq \frac{1}{4}\sigma\sqrt{d} 
 \wedge \|\w_i(0)\| \geq \frac{1}{2}\sigma\sqrt{d} \To \frac{|\langle \w_i(0), \overline{\vv}\rangle|}{\|\w_i(0)\|}\leq \frac{1}{2} \To\frac{\pi}{3}\leq \theta_i(0)\leq \frac{2\pi}{3}. \]

By concentration inequality of Gaussian, $\Pr\left[|\langle \w_i(0), \overline{\vv}\rangle| > \frac{1}{4}\sigma\sqrt{d}\right]\leq 2\exp\left(-\frac{(\frac{1}{4}\sigma\sqrt{d})^2}{2\sigma^2}\right)\leq \frac{\delta}{3n}.$ 
Then $\Pr\left[\theta_i(0)<\frac{\pi}{3}\vee \theta_i(0) >\frac{2\pi}{3}\right]\leq \Pr\left[|\langle \w_i(0), \overline{\vv}\rangle| > \frac{1}{4}\sigma\sqrt{d}\right] + \Pr\left[\|\w_i(0)\| < \frac{1}{2}\sigma\sqrt{d}\right]\leq \frac{2\delta}{3n}.$ By union bound, \eqref{theta(0)} holds with probability at least $1-2\delta/3$. Applying union bound again finishes the proof.
\end{proof}

\subsection{Parameter Setting}\label{valuation}
In this section, we assign values to all intermediate parameters appeared in Theorem \ref{p1}, Theorem \ref{p2} and Theorem \ref{p3}, according to the requirements of these theorems.

\begin{itemize}
    \item First we set $\epsilon_2 = O(n^{-14})$ in Theorem \ref{p2} as required by Theorem \ref{p3}.

    \item Set $\epsilon_1 = O(\epsilon_2^6n^{-1/2})=O(n^{-84.5})$ in Theorem \ref{p1} as required by Theorem \ref{p2}.
    
    \item Set $C=O\left(\frac{\epsilon_1^2}{n}\right)=O(n^{-170})$ in Theorem \ref{p1}.
    
    \item Set $\sigma=O\left(C\epsilon_1^{48}d^{-1/2}\|\vv\|\right)=O\left(\epsilon_1^{50}d^{-1/2}\|\vv\|/n\right)=O\left(n^{-4226}d^{-1/2}\|\vv\|\right)$ in Theorem \ref{p1}.
    
    \item Set $\eta = O\left(\frac{\epsilon_1^2\sigma^2d}{\|\vv\|^2}\right)=O\left(\frac{\sigma^2d}{n^{169}\|\vv\|^2}\right)$ as required by Theorem \ref{p2}. (Note that in Phase 1 and 3, Theorem \ref{p1} and Theorem \ref{p3} also have requirements for $\eta$, but the bound in Theorem \ref{p2} is the tightest one.)
    
    \item Set $T_1=\frac{C}{\eta} = O\left(\frac{\epsilon_1^2}{n\eta}\right)$ in Theorem \ref{p1}. 
    
    \item Finally, set $T_2 = T_1 + \left\lceil\frac{1}{n\eta}\ln\left(\frac{1}{36\epsilon_2}\right)\right\rceil =O\left(\frac{\epsilon_1^2}{n\eta}\right) + O\left(\frac{\log(1/\epsilon_2)}{n\eta}\right)= O\left(\frac{\log(1/\epsilon_2)}{n\eta}\right)=O\left(\frac{\log n}{n\eta}\right)$ in Theorem \ref{p2}.
\end{itemize}

The dependence between these parameters is shown in Figure \ref{dependency graph}.
(We use arrows to indicate dependency, $e.g.$, the arrow from $\epsilon_1$ to $\epsilon_2$ indicates that $\epsilon_1$ depends on $\epsilon_2$.)

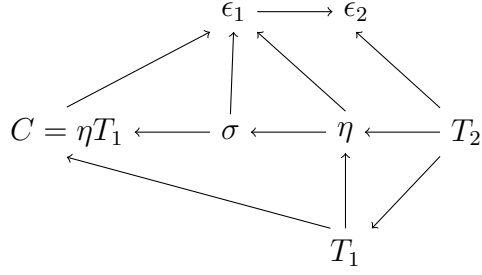
\begin{figure}[t]
\begin{center}
\begin{tikzpicture}
    \node (ep2) {$\epsilon_2$};
    \node (ep) [left =of ep2] {$\epsilon_1$};
    \node (et) [below left=of ep] {$C=\eta T_1$};
    \node (s2) [right=of et] {$\sigma$};
    \node (s1) [right=of s2] {$\eta$};
    \node (t0) [below=of s1] {$T_1$};
    \node (eta) [right=of s1] {$T_2$};

    \draw[->] (ep.east) -- (ep2.west);
    \draw[->] (s2.north) -- (ep.south);
    \draw[->] (et.north) -- ([xshift = 0mm]ep.south west);
    \draw[->] (s1.north) -- (ep.south east);
    \draw[->] (s2.west) -- (et.east);
    \draw[->] (s1.west) -- (s2.east);
    \draw[->] (eta.north west) -- ([xshift = 0mm]ep2.south);
    \draw[->] (eta.west) -- (s1.east);
    \draw[->] ([xshift = -2mm]t0.north) -- (et.south);
    \draw[->] (t0.north) -- (s1.south);
    \draw[->] (eta.south west) -- (t0.north east);
\end{tikzpicture}
\end{center}
    \caption{Parameter Dependency Graph}
    \label{dependency graph}
\end{figure}

\subsection{Non-degeneracy of Student Neurons}\label{Non-degeneracy of Student Neurons}
There is a technical issue in the convergence analysis: if one of the student neuron $\w_i$ is degenerate and $\w_i = \bm{0}$, the loss function $L(\w)$ is not differentiable, hence gradient descent is not well-defined. 

However, our proof shows that such case would not happen and the student neurons are always non-degenerate. Note that the student neuron's norm $\|\w_i\|$ is always lower-bounded in all three phases of our analysis (Phase 1: \eqref{1} in Theorem \ref{p1}, Phase 2: \eqref{2-3} in Theorem \ref{p2}, Phase 3: \eqref{regularization} in Theorem \ref{p3}). By these bounds we have the following corollary describing the non-degeneracy of student neurons.

\begin{corollary}\label{w>0}
If the initialization conditions in Lemma \ref{Initialization} hold, then for $\forall i\in [n], t\in \mathbf{N}$, $\|\w_i(t)\|> {0}$.
\end{corollary}

\begin{remark}
Note that an assumption on the initialization condition such as the one in Corollary \ref{w>0} is necessary, otherwise  there would be counter-examples where all student neurons are degenerate. For example, $\forall c > 0$, if we set $\w_i(0)=-c\vv, \forall i$ and $\eta = \frac{2c}{1+nc}$, then straightforward calculation shows that $\forall i, \nabla_i(0)=-\frac{1+nc}{2}\vv \To \forall i, \w_i(1)= \bm{0}.$. 

\end{remark}

%% file: Content/LowerBound.tex
\section{Lower Bound of the Convergence Rate}\label{Lower Bound of the Convergence Rate}

\subsection{Preliminaries} \label{Implicit Regularization: Gradient Flow Version}
In this section, we do some technical preparations for proving Theorem \ref{lower bound main theorem}.

Taking $\eta\to 0$, we get the gradient flow version of Theorem \ref{p3}.
\begin{theorem}\label{p3flow}
Suppose gradient flow is initialized from a point $\w(0)$ where the conditions in Lemma \ref{Initialization} hold. If $\sigma= O\left(n^{-4226}d^{-1/2}\|\vv\|\right)$, then there exists $T_2 = O\left(\frac{\log n}{n}\right)$ such that $\forall T\in \mathbf{R}_{\geq 0}$ we have
\begin{equation}\label{precondition lower bound 1}
        \frac{4\|\vv\|}{n}\geq \|\w_i(T+T_2)\|\geq \frac{\|\vv\|}{4n}, 
    \end{equation}
and
    \begin{equation}\label{precondition lower bound 2}
        L(T+T_2)\leq O\left(\frac{n^4\|\vv\|^2}{T^3}\right).
    \end{equation}
\end{theorem}

Similarly, we also have the gradient flow version of Corollary \ref{w>0}.
\begin{corollary}\label{w>0 gd flow}
Given that the initial condition in Lemma \ref{Initialization} holds, then for $\forall i\in [n], t\in \mathbf{N}$, $\|\w_i(t)\|> {0}$.
\end{corollary}

\begin{lemma}\label{z>0}
    Given that the initial condition in Lemma \ref{Initialization} holds, and the initialization is non-degenerate, then $\forall t, \exists i\in [n], ~s.t.~ \z_i(t)\neq \bm{0}$.
\end{lemma}
\begin{proof}
Assume for contradiction that $\exists t \in \mathbb{R}^+$ such that $\z_1(t)=\cdots=\z_n(t)=\bm{0}$. Define 
    \[\overline{t}:=\inf \{t | \z_1(t)=\cdots=\z_n(t)=\bm{0}\}.\]

Then the continuity of $\z_i$ implies $\z_1(\overline{t})=\cdots=\z_n(\overline{t})=\bm{0}$, so $\overline{t} > 0$. On the other hand, Corollary \ref{w>0} indicates that $\w_i(\overline{t})\neq \bm{0}, \forall i\in[n]$. Since $\w_i$ is continuous, there exists a neighborhood of $\overline{t}$ such that for $\forall i,j\in[n]$, $\|\w_i(t)\|/\|\w_j(t)\|$ and $\|\vv\|/\|\w_i(t)\|$ are bounded by a fixed constant when $t$ is in this neighborhood. Furthermore, since $\overline{t}>0$, $\exists \epsilon>0$ and constant $C>1$ such that for $\forall t\in[\overline{t}-\epsilon, \overline{t}], \forall i\in[n]$ we have

\[\left|\pi+\sum_{j\neq i}\frac{\|\w_j(t)\|}{\|\w_i(t)\|}\sin\theta_{ij}(t)-\frac{\|\vv\|}{\|\w_i(t)\|}\sin\theta_i(t)\right|\leq \pi +\sum_{j\neq i}\frac{\|\w_j(t)\|}{\|\w_i(t)\|}|+\frac{\|\vv\|}{\|\w_i(t)\|}\leq C.\]

Then in the interval $[\overline{t}-\epsilon, \overline{t}]$ we have
\begin{gather*}
 \left\|\frac{\partial \z_i}{\partial t}\right\|=\left\|-\frac{1}{2\pi}\left(\pi+\sum_{j\neq i}\frac{\|\w_j\|}{\|\w_i\|}\sin\theta_{ij}-\frac{\|\vv\|}{\|\w_i\|}\sin\theta_i\right)\z_i-\sum_{j\neq i} \frac{\pi-\theta_{ij}}{2\pi}\z_j\right\|\leq C\|\z_i\|+\sum_{j\neq i}\|\z_j\|,\\
\To \frac{\partial \|\z_i\|^2}{\partial t}=2\left\langle\frac{\partial \z_i}{\partial t}, {\z_i}\right\rangle\geq -2C\|\z_i\|\left(\sum_{j\in[n]}\|\z_j\|\right)\\
\To \frac{\partial}{\partial t} \sum_{j\in[n]}\|\z_j\|^2\geq -2C\left(\sum_{j\in[n]}\|\z_j\|\right)^2\geq -2nC\left(\sum_{j\in[n]}\|\z_j\|^2\right)\\
\To \frac{\partial}{\partial t}\left[e^{2nCt}\left(\sum_{j\in[n]}\|\z_j\|^2\right)\right]=e^{2nCt}\left[2nC\left(\sum_{j\in[n]}\|\z_j\|^2\right)+\frac{\partial}{\partial t}\left(\sum_{j\in[n]}\|\z_j\|^2\right)\right]\geq 0\\
\To \sum_{j\in[n]}\|\z_j(\overline{t})\|^2\geq e^{-2nC\epsilon}\left(\sum_{j\in[n]}\|\z_j(\overline{t}-\epsilon)\|^2\right).
\end{gather*}

Here the $(t)$ indicator is omitted for simplicity. Note that we bound $\frac{\partial \|\z_i\|^2}{\partial t}$ instead of $\frac{\partial \|\z_i\|}{\partial t}$ here, since $\|\z_i\|$ might not be differentiable if $\z_i=\bm{0}$, while $\|\z_i\|^2$ is always differentiable.

Finally, due to the definition of $\overline{t}$, there exists $i\in [n]$ such that $\z_i(\overline{t}-\epsilon)\neq \bm{0}$. Then $\sum_{j\in[n]}\|\z_j(\overline{t})\|^2\geq e^{-2nC\epsilon}\left(\sum_{j\in[n]}\|\z_j(\overline{t}-\epsilon)\|^2\right)>0$, a contradiction.
\end{proof}

\subsection{Proofs for Section \ref{Proof Sketch: Convergence Rate Lower Bound}}

By the closed form formula of gradient \eqref{gradient}, the dynamics of $\z_i$ is given by

\begin{equation}\label{dynamics-z}
    \frac{\partial \z_i}{\partial t}=-\frac{1}{2\pi}\left(\pi+\sum_{j\neq i}\frac{\|\w_j\|}{\|\w_i\|}\sin\theta_{ij}-\frac{\|\vv\|}{\|\w_i\|}\sin\theta_i\right)\z_i-\sum_{j\neq i} \frac{\pi-\theta_{ij}}{2\pi}\z_j.
\end{equation}

\begin{replemma}{max_kappa}
    If there exists $i, j$ such that $\kappa_{ij}(t)=\kappa_{\max}(t) < \frac{\pi}{2}$, then $\cos\kappa_{ij}(t)$ is well-defined in an open neighborhood of $t$, differentiable at $t$, and
    \[\frac{\partial}{\partial t} \cos\kappa_{ij}(t)\leq -\frac{\pi-\theta_{ij}(t)}{\pi}(1-\cos\kappa_{ij}^2(t)).\]
\end{replemma}

\begin{proof}
First note that for $\forall i\in Q^+(t)$, $\|\z_i(t)\|>0 \To \overline{\z_i(t)}=\frac{\z_i(t)}{\|\z_i(t)\|}$ is differentiable at $t$. Therefore, for $\forall i,j \in Q^+(t)$, $\cos\kappa_{ij}=\langle \overline{\z_i(t)}, \overline{\z_j(t)}\rangle$ is well-defined in an open neighborhood of $t$ and differentiable at $t$. 
According to the dynamics of $\z_i$ \eqref{dynamics-z}, $\forall i\in Q^+(t)$ we have
\[\begin{split}
    \frac{\partial}{\partial t}\overline{\z_i}
    =\frac{\|\z_i\|\frac{\partial \z_i}{\partial t}-\frac{\partial \|\z_i\|}{\partial t}\z_i}{\|\z_i\|^2}
    =\frac{\frac{\partial \z_i}{\partial t}-\langle\frac{\partial \z_i}{\partial t}, \overline{\z_i}\rangle\overline{\z_i}}{\|\z_i\|}
    =-\sum_{k\neq i}\frac{\pi-\theta_{ik}}{2\pi}\frac{\z_k-\langle\z_k,\overline{\z_i}\rangle\overline{\z_i}}{\|\z_i\|}.
\end{split}\]

Then \begin{equation}\label{kappa dynamics}
    \begin{split}
    &\frac{\partial \cos\kappa_{ij}}{\partial t}
    =\frac{\partial \langle \overline{\z_i}, \overline{\z_j}\rangle}{\partial t}\\
    =&\left\langle\frac{\partial   }{\partial t}\overline{\z_i},\overline{\z_j}\right\rangle
    +\left\langle\overline{\z_i}, \frac{\partial}{\partial t}\overline{\z_j}\right\rangle\\
    =&-\sum_{k\neq i, k\in Q^+} \frac{\pi - \theta_{ik}}{2\pi}\frac{\|\z_k\|}{\|\z_i\|}(\langle \overline{\z_k}, \overline{\z_j}\rangle-\langle \overline{\z_k}, \overline{\z_i}\rangle\langle \overline{\z_i}, \overline{\z_j}\rangle)
    -\sum_{k\neq j, k\in Q^+} \frac{\pi - \theta_{jk}}{2\pi}\frac{\|\z_k\|}{\|\z_j\|}(\langle \overline{\z_k}, \overline{\z_i}\rangle-\langle \overline{\z_k}, \overline{\z_j}\rangle\langle \overline{\z_i}, \overline{\z_j}\rangle)\\
    =&\underbrace{-\frac{\pi - \theta_{ij}}{2\pi}\left(\frac{\|\z_i\|}{\|\z_j\|}+\frac{\|\z_j\|}{\|\z_i\|}\right)(1-\cos^2\kappa_{ij})}_{I_1}\\
    &\underbrace{-\sum_{k\neq i,j\wedge k\in Q^+}\left[ \frac{\pi - \theta_{ik}}{2\pi}\frac{\|\z_k\|}{\|\z_i\|}(\cos\kappa_{kj}-\cos\kappa_{ki}\cos\kappa_{ij})+
    \frac{\pi - \theta_{jk}}{2\pi}\frac{\|\z_k\|}{\|\z_j\|}(\cos\kappa_{ik}-\cos\kappa_{kj}\cos\kappa_{ij})\right]}_{I_2}.
\end{split}
\end{equation}

The expression above splits into two terms $I_1$ and $I_2$. The most important observation is that, by setting $\overline{\z_i}, \overline{\z_j}$ to be the pair of maximally separated vectors, i.e., $\kappa_{ij}=\kappa_{\max}$, the second term $I_2$ is guaranteed to be nonpositive. This is because for $\forall k,$
\[\kappa_{ki}\leq \kappa_{\max}<\pi/2, \kappa_{kj}\leq \kappa_{\max}=\kappa_{ij}
\To \cos\kappa_{kj}\geq \cos\kappa_{ij}\geq \cos\kappa_{ki}\cos\kappa_{ij} 
\To \cos\kappa_{kj}-\cos\kappa_{ki}\cos\kappa_{ij}\geq 0,\]
similarly we have $\cos\kappa_{ik}-\cos\kappa_{kj}\cos\kappa_{ij}\geq 0, \forall k$. So $I_2\leq 0$ when $\kappa_{ij} =\kappa_{\max}$. This implies that when $\kappa_{ij}=\kappa_{\max}$, 
\[\frac{\partial \cos\kappa_{ij}}{\partial t}\leq I_1=-\frac{\pi - \theta_{ij}}{2\pi}\left(\frac{\|\z_i\|}{\|\z_j\|}+\frac{\|\z_j\|}{\|\z_i\|}\right)(1-\cos^2\kappa_{ij})\leq -\frac{\pi - \theta_{ij}}{\pi}(1-\cos^2\kappa_{ij}).\]

\end{proof}

\begin{lemma}\label{separation}
Given that the initial condition in Lemma \ref{Initialization} holds, suppose the network is over-parameterized, i.e., $n\geq 2$, and the initialization is non-degenerate, then for $\forall t\in\mathbb{R}_{\geq 0}$, at least one of the following two conditions must hold:
\begin{gather}
    \exists i\in[n] ~s.t.~ \z_i(t)=\bm{0},\\
    \kappa_{\max}(t)\geq \frac{\kappa_{\max}(0)}{3}.
\end{gather}
\end{lemma}

\begin{proof}
Assume for contradiction that $\exists t$ such that $\z_i(t)\neq \bm{0}, \forall i$ and $\kappa_{\max}(t)<\frac{\kappa_{\max}(0)}{3}$. Then we can define
\[t^*=\inf\left\{t\in\mathbb{R}|\forall i, \z_i(t)\neq \bm{0} \wedge \kappa_{\max}(t)<\frac{\kappa_{\max}(0)}{3}\right\}.\]
Note that by lemma \ref{z>0} we have $Q^+(t^*)\geq 1$. For $\forall i,j\in Q^+(t^*)$, if $\kappa_{ij}(t^*)>\kappa_{\max}(0)/3$, due to the continuity of $\kappa_{ij}$, $\kappa_{ij} >\kappa_{\max}(0)/3$ holds in an open neighborhood of $t^*$, which contradicts the definition of $t^*$. So $\forall i,j\in Q^+(t^*)$ we have $\kappa_{ij}(t^*)\leq \kappa_{\max}(0)/3$.

\noindent \textbf{Step 1: } First we prove $\forall i, \z_i(t^*)\neq \bm{0}.$

If $\exists i$ such that $\z_i(t^*)=\bm{0}$, then for such $i$ we have $\frac{\partial \z_i(t^*)}{\partial t}=-\sum_{j\neq i}\frac{\pi-\theta_{ij}(t^*)}{2\pi}\z_j(t^*).$ Since $Q^+(t^*)\neq \varnothing$, pick $k\in Q^+(t^*)$ and we have
\[\left\langle\frac{\partial \z_i(t^*)}{\partial t}, \z_k(t^*)\right\rangle=-\sum_{j\neq i\wedge j\in Q^+(t^*)}\frac{\pi-\theta_{ij}(t^*)}{2\pi}\langle\z_j(t^*), \z_k(t^*)\rangle<0,\]
where the last inequality is because $\kappa_{jk}(t^*)\leq \kappa_{\max}(0)/3<\frac{\pi}{2}$.

On the other hand, the definition of $\frac{\partial \z_i(t^*)}{\partial t}$ implies that $\exists \epsilon>0$, $\forall t'\in [t^*, t^*+\epsilon)$, $\z_i(t')=\z_i(t^*)+(t'-t^*)\frac{\partial \z_i(t^*)}{\partial t}+\bm{o}(t'-t^*)=(t'-t^*)\frac{\partial \z_i(t^*)}{\partial t}+\bm{o}(t'-t^*)$. Similarly, $\z_k(t')=\z_k(t^*)+(t'-t^*)\frac{\partial \z_k(t^*)}{\partial t}+\bm{o}(t'-t^*)$. Then 
\begin{equation}\label{angle-0}
    \left\langle \z_i(t'), \z_k(t')\right\rangle=(t'-t^*)\left\langle\frac{\partial \z_i(t^*)}{\partial t}, \z_k(t^*)\right\rangle+o(t'-t^*).
\end{equation}
Since $\left\langle\frac{\partial \z_i(t^*)}{\partial t}, \z_k(t^*)\right\rangle<0$ is a negative constant, there exists $\epsilon'>0$ such that for $\forall t'\in[t^*, t^*+\epsilon')$, \eqref{angle-0} is negative, consequently $\kappa_{ik}(t')=\arccos\left(\left\langle \overline{\z_i(t')}, \overline{\z_k(t')}\right\rangle\right)>\frac{\pi}{2}.$ So $\forall t'\in[t^*, t^*+\epsilon')$, $\kappa_{\max}(t')\geq \kappa_{ik}(t')>\pi/2\geq \kappa_{\max}(0)/3$, this contradicts the definition of $t^*$.

\noindent \textbf{Step 2: }After proving $\z_i(t^*)\neq \bm{0}, \forall i$ and $\kappa_{ij}(t^*)\leq \kappa_{\max}(0)/3<\kappa_{\max}(0), \forall i,j\in[n]$, we aim to derive a contradiction.

Note that $t^*\neq 0$ due to the definition of $\kappa_{\max}$. By the continuity of $\z_i$, $\exists \epsilon_1>0$ such that for $\forall t\in (t^*-\epsilon_1,t^*+\epsilon_1), i\in[n],$ $\z_i(t)\neq \bm{0}$. Then the definition of $t^*$ implies that $\forall t\in(t^*-\epsilon_1, t^*)$, $\kappa_{\max}(t)\geq \kappa_{\max}(0)/3$. Since on the interval $(t^*-\epsilon_1,t^*+\epsilon_1)$, $\kappa_{\max}=\max_{i,j\in[n]}\kappa_{ij}$ is continuous
\footnote{Generally $\kappa_{\max}$ may not be continuous since the range of taking the max ($i,j\in Q^+$) might change, but it is continuous on $(t^*-\epsilon_1,t^*+\epsilon_1)$ since all $\z_i$'s are nonzero on this interval.}, we have that $\kappa_{\max}(t^*)\geq \kappa_{\max}(0)/3$. Then $\kappa_{\max}(t^*)=\kappa_{\max}(0)/3$. Pick $i,j$ such that $\kappa_{ij}(t^*)=\kappa_{\max}(t^*)$. Note that $\theta_{ij}(t^*)<\pi$, otherwise $\kappa_{ij}(t^*)=\pi$, a contradiction. Then by lemma \ref{max_kappa} we have \[\frac{\partial}{\partial t} \cos\kappa_{ij}(t^*)\leq -\frac{\pi-\theta_{ij}(t^*)}{\pi}(1-\cos\kappa_{ij}^2(t^*))<0.\]
So $\frac{\partial}{\partial t}\kappa_{ij}(t^*)>0 \To \exists \epsilon_2> 0 ~s.t.~ \forall t\in (t^*, t^*+\epsilon_2), \kappa_{\max}(t)\geq \kappa_{ij}(t)>\kappa_{ij}(t^*)=\kappa_{\max}(t^*)=\kappa_{\max}(0)/3,$
this contradicts the definition of $t^*$.
\end{proof}

\begin{lemma}\label{Z>}
   Given that the initial condition in Lemma \ref{Initialization} holds, suppose the network is over-parameterized, i.e., $n\geq 2$, and the initialization is non-degenerate, then for $\forall t\in \mathbb{R}_{\geq 0}$ we have
    \[Z(t)\geq \Omega(\kappa_{\max}(0)\max_{i\in [n]}\|\z_i(t)\|).\]
\end{lemma}

\begin{proof}
    We show that for $\forall t\in\mathbb{R}_{\geq 0},$
    \begin{equation}\label{max-zij}
        \max_{i,j\in[n]} \|\z_i(t)-\z_j(t)\|\geq \Omega(\kappa_{\max}(0)\max_{i\in[n]} \|\z_i(t)\|).
    \end{equation} 
    W.L.O.G., suppose $\z_1(t)=\max_{i\in[n]} \|\z_i(t)\|$.
By lemma \ref{separation}, for $\forall t$ one of the following two cases must happen:
\begin{itemize}
    \item $\exists k ~s.t.~ \z_k(t)=\bm{0}$. 
    
    By lemma \ref{z>0}, $k\neq 1$. Then $\max_{i,j\in[n]} \|\z_i(t)-\z_j(t)\|\geq \|\z_1(t)-\z_k(t)\|=\|\z_1(t)\|\geq O(\kappa_{\max}(0)\max_{i\in[n]} \|\z_i(t)\|)$.

    \item $\kappa_{\max}(t)\geq \kappa_{\max}(0)/3$.

    Pick a pair $i,j$ such that $\kappa_{ij}(t)=\kappa_{\max}(t)$. Then $\kappa_{1i}(t)+\kappa_{1j}(t)\geq \kappa_{ij}(t)\geq \kappa_{\max}(0)/3\To \max\{\kappa_{1i}(t), \kappa_{1j}(t)\}\geq \kappa_{\max}(0)/6$. W.L.O.G., suppose $\kappa_{1i}(t) \geq \kappa_{\max}(0)/6$. If $\kappa_{1i}(t)\leq \pi/2$, then $\|\z_1(t)-\z_i(t)\|\geq \|\z_1(t)\|\sin \kappa_{1i}(t)\geq \Omega(\kappa_{\max}(0)\max_{i\in[n]} \|\z_i(t)\|)$. If $\kappa_{1i}(t)>\pi/2$, then $\|\z_1(t)-\z_i(t)\|\geq \|\z_1(t)\|\geq \Omega(\kappa_{\max}(0)\max_{i\in[n]} \|\z_i(t)\|)$. So no matter which case happens, \eqref{max-zij} always holds.
\end{itemize}

In conclusion, we have $Z(t)\geq \max_{i,j\in[n]}\|\z_i(t)-\z_j(t)\|\geq \Omega(\kappa_{\max}(0)\max_{i\in[n]} \|\z_i(t)\|).$
\end{proof}

Combined with Lemma \ref{z>0}, Lemma \ref{Z>} immediately implies the following corollary.
\begin{corollary}\label{Z(t)>0}
Given that the initial condition in Lemma \ref{Initialization} holds,  suppose $\z_i(0)\neq 0, \forall i\in[n]$, $\kappa_{\max}(0)>0$ and $n\geq 2$, then for $\forall t\in \mathbb{R}_{\geq 0}$ we have $Z(t)>0$.
\end{corollary}

\begin{lemma}\label{partial Z>}
Suppose the conditions \eqref{precondition lower bound 1} \eqref{precondition lower bound 2} in Theorem \ref{p3flow} holds. Suppose the network is over-parameterized, i.e., $n\geq 2$, and the initialization is non-degenerate. Then $\forall t\geq T_2$ we have
\[\frac{\partial }{\partial t}Z(t)\geq -O(n^2\|\vv\|\theta_{\max}^2(t)).\]
\end{lemma}

\begin{proof}

Recall the closed form expression of gradient \eqref{gradient}, which can be decomposed into two terms,
\[\frac{\partial L(\w)}{\partial \w_i}=\underbrace{\frac{1}{2}\left(\sum_j \w_j -\bm{v}\right)}_{I_1}
+\underbrace{\frac{1}{2\pi}\left[\left(\sum_{j\neq i}\|\w_j\|\sin\theta_{ij}-\|\vv\|\sin\theta_i\right)\overline{\w}_i-\sum_{j\neq i} \theta_{ij} \w_j +\theta_i\vv\right]}_{I_2}.\]

The second term $I_2$ can be rewritten as
\[I_2=\frac{1}{2\pi}\left[\sum_{j\neq i} \|\w_j\|(\sin\theta_{ij}-\theta_{ij})\overline{\w}_i+\sum_{j\neq i}\|\w_j\|\theta_{ij}(\overline{\w}_i-\overline{\w}_j)-\|\vv\|(\sin\theta_i-\theta_i)\overline{\w}_i-\|\vv\|\theta_i(\overline{\w}_i-\overline{\vv})\right]\]

By Theorem \ref{p3flow}, for $\forall t>T_2$, we have $\|\w_i(t)\|=\Theta(\|\vv\|/n)$. Note that $\sin\theta_{ij}(t)-\theta_{ij}(t)=O(\theta_{ij}^3(t))=O(\theta_{\max}^3(t))$, similarly $\sin\theta_i(t)-\theta_i(t)=O(\theta_{\max}^3(t))$. Combined with $\|\overline{\w}_i-\overline{\w}_j\|=2\sin(\theta_{ij}/2)=O(\theta_{\max})$, we get
\[\begin{split}
    \|I_2\|\leq \sum_{j\neq i}\Theta(\|\vv\|/n)O(\theta^3_{\max})+\sum_{j\neq i}\Theta(\|\vv\|/n)\theta_{\max}O(\theta_{\max})\\
    +\|\vv\|O(\theta_{\max}^3)+\|\vv\|\theta_{\max}O(\theta_{\max})=O(\|\vv\|\theta_{\max}^2).
\end{split}\]

Then $\forall i, \frac{\partial \w_i}{\partial t}=-\frac{\partial L(\w)}{\partial \w_i}=-\frac{1}{2}\left(\sum_j \w_j -\bm{v}\right)+\bm{O}(\|\vv\|\theta_{\max}^2)$. Note that the first term $-\frac{1}{2}\left(\sum_j \w_j -\bm{v}\right)$ is the same for all $\w_i$, so for $\forall i,j\in[n]$ we have $\left\|\frac{\partial \w_i-\w_j}{\partial t}\right\|=\left\|\frac{\partial \w_i}{\partial t}-\frac{\partial \w_j}{\partial t}\right\|={O}(\|\vv\|\theta_{\max}^2)$.

For $\forall i,j \in [n]$, if $\z_i(t)=\z_j(t)$ then $\frac{\partial \z_i(t)}{\partial t}=\frac{\partial \z_j(t)}{\partial t}\To \frac{\partial }{\partial t}\|\z_i(t)-\z_j(t)\|=0$. Otherwise $\z_i(t)-\z_j(t)\neq \bm{0}$ and
\[\begin{split}
\frac{\partial}{\partial t}\left( \|\z_i(t)-\z_j(t)\|\right)
&=\left\langle\frac{\partial}{\partial t}\left( \z_i(t)-\z_j(t)\right), \overline{\z_i(t)-\z_j(t)}\right\rangle 
\geq -\left\|\frac{\partial}{\partial t} (\z_i(t)-\z_j(t))\right\|\\
&\geq -\left\|\frac{\partial}{\partial t} (\w_i(t)-\w_j(t))\right\|\geq -O(\|\vv\|\theta_{\max}^2(t)).
\end{split}\]

So for both cases we have $\frac{\partial}{\partial t}\left( \|\z_i(t)-\z_j(t)\|\right)\geq -O(\|\vv\|\theta_{\max}^2(t))$. 

Then $\frac{\partial}{\partial t}Z(t)=\sum_{1\leq i <j\leq n} \frac{\partial}{\partial t}\left( \|\z_i(t)-\z_j(t)\|\right)\geq -O(n^2\|\vv\|\theta_{\max}^2(t)).$
\end{proof}

\subsection{Proof of Main Theorem}\label{Proof of Main Theorem: lower bound}
\begin{reptheorem}{lower bound main theorem}
Suppose the network is over-parameterized, i.e., $n\geq 2$. For $\forall \delta >0$, if the initialization is non-degenerate, $d=\Omega(\log(n/\delta))$, $\sigma= O\left(n^{-4226}d^{-1/2}\|\vv\|\right)$, then there exists $T_2 = O\left(\frac{\log n}{n}\right)$ such that with probability at least $1-\delta$, for $\forall t\geq T_2$ we have
\[L(\w(t))^{-1/3}\leq O\left(\frac{n^{17/3}}{\kappa_{\max}^2(0)\|\vv\|^{2/3}}\right)(t-T_2)+\gamma,\]
where $\gamma\in\mathbf{R}^+$ is a constant that does not depend on $t$.
\end{reptheorem}

\begin{proof}
For $\forall t \geq T_2$ we have $\max_{i\in[n]}\|\z_i(t)\|\geq \theta_{\max}(t)\Theta(\|\vv\|/n) $. Then by lemma \ref{Z>}, for $\forall t \geq T_2$, $Z(t)\geq \Omega(\kappa_{\max}(0)\theta_{\max}(t)\|\vv\|/n)
\To \theta_{\max}(t)=O\left(\frac{nZ(t)}{\kappa_{\max}(0)\|\vv\|}\right).$ Combined with lemma \ref{partial Z>} we have
\begin{equation}\label{z dynamics}
    \frac{\partial }{\partial t}Z(t)\geq -O(n^2\|\vv\|\theta_{\max}^2(t))\geq -O\left(\frac{n^4Z^2(t)}{\kappa_{\max}^2(0)\|\vv\|}\right).
\end{equation}

By Corollary \ref{Z(t)>0}, $Z(t)$ is always strictly positive. We can therefore calculate the dynamics of $1/Z(t)$ as: $\forall t \geq T_2$,
\begin{equation}\label{Z(t)}
    \frac{\partial }{\partial t}\frac{1}{Z(t)}=-\frac{1}{Z^2(t)}\frac{\partial }{\partial t}Z(t)\leq O\left(\frac{n^4}{\kappa_{\max}^2(0)\|\vv\|}\right)\To \frac{1}{Z(t)}= O\left(\frac{n^4}{\kappa_{\max}^2(0)\|\vv\|}(t-T_2)\right)+\frac{1}{Z(T_2)}.
\end{equation}

On the other hand, by Theorem \ref{p3flow} we have \[Z(t)\leq \sum_{1\leq i<j\leq n} (\|\z_i(t)\|+\|\z_j(t)\|)\leq \sum_{1\leq i<j\leq n} (\theta_i(t)\|\w_i(t)\|+\theta_j(t)\|\w_j(t)\|)\leq O(n\|\vv\|\theta_{\max}).\] 

By lemma \ref{theta} we have 
\begin{equation}\label{Z vs L}
\forall t\geq T_2, \theta_{\max}(t)=O\left(\left(\frac{L(\w(t))n^2}{\|\vv\|^2}\right)^{1/3}\right)\To 
L(\w(t))\geq \Omega\left(\frac{Z^3(t)}{n^5\|\vv\|}\right).
\end{equation}
Combined with \eqref{Z(t)}, we have
\[
    L(\w(t))^{-1/3}\leq O\left(\frac{n^{17/3}}{\kappa_{\max}^2(0)\|\vv\|^{2/3}}\right)(t-T_2)+\frac{n^{5/3}\|\vv\|^{1/3}}{Z(T_2)}.
\]

Finally, since Corollary \ref{Z(t)>0} implies $Z(T_2)>0$, setting $\gamma=\frac{n^{5/3}\|\vv\|^{1/3}}{Z(T_2)}$ finishes the proof.
\end{proof}